\documentclass{article}

\PassOptionsToPackage{numbers, compress}{natbib}
\usepackage[preprint]{neurips_2020}

\usepackage[utf8]{inputenc} % allow utf-8 input
\usepackage[T1]{fontenc}    % use 8-bit T1 fonts
\usepackage{hyperref}       % hyperlinks
\usepackage{url}            % simple URL typesetting
\usepackage{booktabs}       % professional-quality tables
\usepackage{amsfonts}       % blackboard math symbols
\usepackage{nicefrac}       % compact symbols for 1/2, etc.
\usepackage{microtype}      % microtypography
\usepackage{graphicx}
\usepackage{subfigure}
\usepackage{amsthm}
\usepackage{amsfonts}
\usepackage{amsmath}
\usepackage{amssymb}
\usepackage{makecell}
\newtheorem{theorem}{Theorem}

\newcommand*\samethanks[1][\value{footnote}]{\footnotemark[#1]}

\title{Explaining Explanations: Axiomatic Feature Interactions for Deep Networks}

\author{%
  Joseph D. Janizek\thanks{These authors contributed equally to this work and are listed alphabetically.} \\
  Paul G. Allen School of Computer Science\\
  University of Washington\\
  Seattle, WA 98144 \\
  \texttt{jjanizek@cs.washington.edu} \\
   \And
  Pascal Sturmfels \samethanks \\
  Paul G. Allen School of Computer Science\\
  University of Washington\\
  Seattle, WA 98144 \\
  \texttt{psturm@cs.washington.edu} \\
   \AND
  Su-In Lee \samethanks \\
  Paul G. Allen School of Computer Science\\
  University of Washington\\
  Seattle, WA 98144 \\
  \texttt{suinlee@cs.washington.edu} \\
}

\begin{document}

\maketitle

\begin{abstract}
  Recent work has shown great promise in explaining neural network behavior. In particular, feature attribution methods explain which features were most important to a model's prediction on a given input. However, for many tasks, simply knowing which features were important to a model's prediction may not provide enough insight to understand model behavior. The \textit{interactions} between features within the model may better help us understand not only the model, but also why certain features are more important than others. In this work, we present Integrated Hessians\footnote[2]{Code available at \url{https://github.com/suinleelab/path_explain}}: an extension of Integrated Gradients that explains pairwise feature interactions in neural networks. Integrated Hessians overcomes several theoretical limitations of previous methods to explain interactions, and unlike such previous methods is not limited to a specific architecture or class of neural network. Additionally, we find that our method is faster than existing methods when the number of features is large, and outperforms previous methods on existing quantitative benchmarks.
%   We apply Integrated Hessians on a variety of neural networks trained on language data, biological data, astronomy data, and medical data and gain new insight into model behavior in each domain.
\end{abstract}

\section{Introduction and Prior Work}
\label{sec:Intro}

Deep neural networks are one of the most popular classes of machine learning model. They have achieved state-of-the-art performance in problem domains ranging from natural language processing to image recognition \citep{devlin2018bert, he2016deep}. They have even outperformed other non-linear model types on structured tabular data \citep{shavitt2018regularization}. Because neural networks have been traditionally difficult to interpret compared to simpler model classes, gaining a better understanding of their predictions is desirable for a variety of reasons. To the extent that these algorithms are used in automated decisions impacting humans, explanations may be legally required \citep{righttoexplanation}. When used in high stakes applications, it is essential to ensure that models are making safe decisions for the right reasons \citep{geis2019ethics}. During model development, interpretability methods can help debug undesirable model behavior \citep{sundararajan2017axiomatic}.

\textbf{Feature attribution:} There have been a large number of recent approaches to interpret deep neural networks, ranging from methods that aim to distill complex models into more simple models \citep{tan2018learning, wu2018beyond, puri2017magix}, to methods that aim to identify the most important concepts learned by a network \citep{kim2017interpretability, olah2018building, olah2017feature, fong2018net2vec, erhan2009visualizing, mahendran2015understanding}. One of the best-studied sets of approaches is known as \emph{feature attribution methods} \citep{binder2016layer,shrikumar2017learning,lundberg2017unified,ribeiro2016should}. These approaches explain a model's prediction by assigning credit to each input feature based on how much it influenced that prediction. Although these approaches help practitioners understand which features are important, they do not explain \textit{why} certain features are important or how features interact in a model. In order to develop a richer understanding of model behavior, it is therefore desirable to develop methods to explain not only feature attributions but also feature interactions. For example, in Figure \ref{fig:not_bad_comp}, we show that word-level interactions can help us distinguish why deeper, more expressive neural networks outperform simpler ones on language tasks.

\textbf{Feature interaction:} There are several existing methods that explain feature interactions in neural networks. \citet{cui2019recovering} propose a method to explain global interactions in Bayesian Neural Networks (BNN) by examining pairs of features that have large second-order derivatives at the input. Neural Interaction Detection is a method that detects statistical interactions between features by examining the weight matrices of feed-forward neural networks \citep{tsang2017detecting}. Furthermore, several authors have proposed domain-specific methods for finding interactions in the area of deep learning for genomics \citep{koo2018inferring, greenside2018discovering}. For example, Deep Feature Interaction Maps detect interactions between two features by calculating the change in the attribution of one feature incurred by changing the value of the second \cite{greenside2018discovering}. \citet{singh2018hierarchical} propose a generalization of Contextual Decomposition \citep{murdoch2018beyond} that can explain interactions for feed-forward and convolutional architectures. In game theory literature, \citet{grabisch1999axiomatic} propose the Shapley Interaction Index, which allocates credit to interactions between players in a coalitional game by considering all possible subsets of players.

\textbf{Limitations of Prior Approaches:} While previous approaches have taken important steps towards understanding feature interaction in neural networks, all suffer from practical limitations, including being limited to specific types of architectures. Neural Interaction Detection only applies to feed-forward neural network architectures, and can not be used on networks with convolutions, recurrent units, or self-attention. Contextual Decomposition has been applied to LSTMs, feed-forward neural networks and convolutional networks, but to our knowledge is not straightforward to apply to more recent innovations in deep learning, such as self-attention layers. The approach suggested by \citet{cui2019recovering} is limited in that it requires the use of Bayesian Neural Networks; it is unclear how to apply the method to standard neural networks. Deep Feature Interaction Maps only work when the input features for a model have a small number of discrete values (such as genomic sequence data). The Shapley Interaction Index, like the Shapley Value, is NP-hard to compute exactly \cite{elkind2009computational}. 

Furthermore, most existing methods to detect interactions do not satisfy the common-sense axioms that have been proposed by feature attribution methods \citep{sundararajan2017axiomatic, lundberg2017unified}. This leads these previous approaches being provably unable to find learned interactions, or more generally finding counter-intuitive interactions (see \autoref{sec:SimulatedXOR}). Existing methods that do satisfy such axioms, such as those based on the Shapley Interaction Index \citep{grabisch1999axiomatic,dhamdhere2019shapley}, are computationally inefficient to compute or even approximate.

\textbf{Our Contributions:} (1) We propose an approach to quantify pairwise feature interactions that can be applied to any neural network architecture; (2) We identify several common-sense axioms that feature-level interactions should satisfy and show that our proposed method satisfies them; (3) We provide a principled way to compute interactions in ReLU-based networks, which are piece-wise linear and have zero second derivatives; (4) We evaluate our method against existing methods and show that it better identifies interactions in simulated data.

\begin{figure*}
    \centering
    \includegraphics[width=0.65\paperwidth]{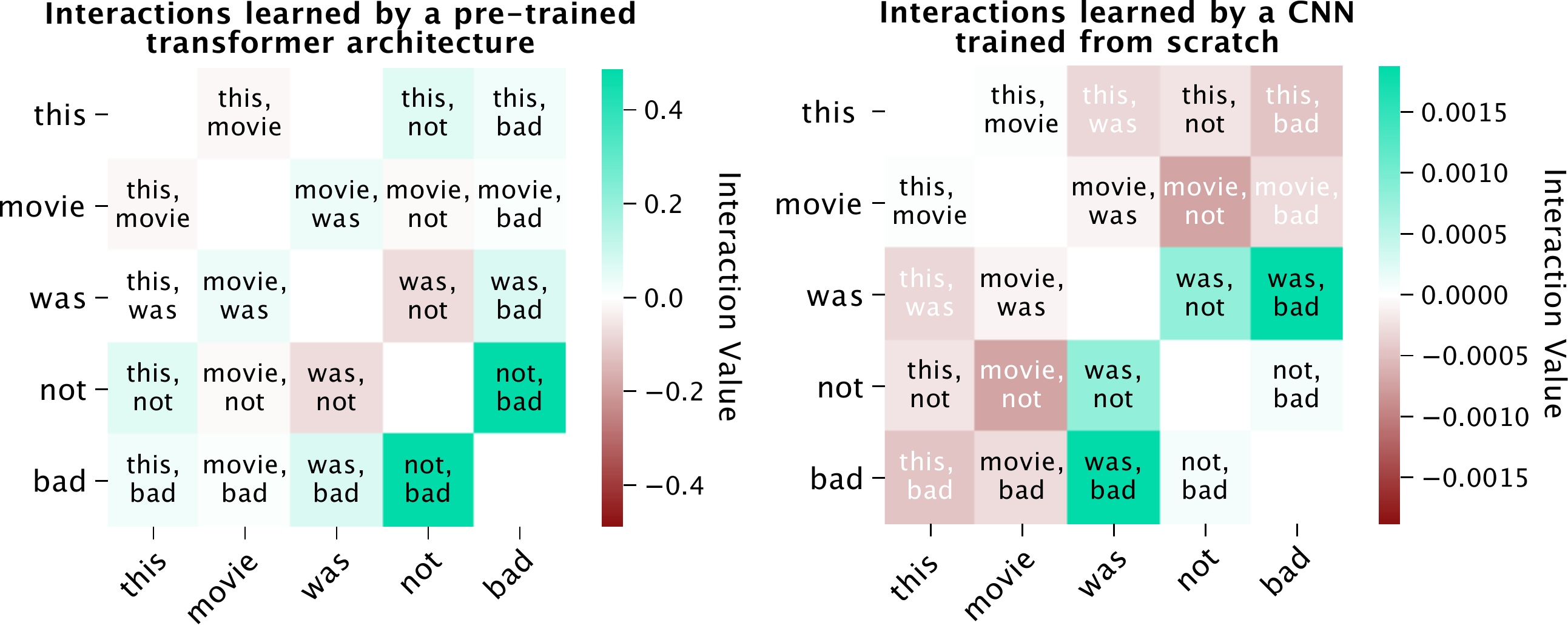}
    \caption{ Interactions help us understand why certain models perform better than others. Here, we examine interactions on the sentence ``this movie was not bad." We compare two models trained to do sentiment analysis on the Stanford Sentiment dataset: a pre-trained transformer, DistilBERT (left), which predicts positive with 98.2\% confidence, and a convolutional neural network trained from scratch (right), which predicts negative with 97.6\% confidence. The transformer picks up on negation patterns: ``not bad'' has a positive interaction, despite the word ``bad'' being negative. The CNN mostly picks up on negative interactions like ``movie not'' and ``movie bad''. }
    \label{fig:not_bad_comp}
\end{figure*}

\section{Explaining Explanations with Integrated Hessians}
\label{sec:ExpectedHessians}

To derive our feature interaction values, we start by considering Integrated Gradients (IG), proposed by \citet{sundararajan2017axiomatic}. We represent our model as a function $f : \mathbb{R}^{d} \mapsto \mathbb{R}$.\footnote{In the case of multi-output models, such as multi-class classification problems, we assume the function is indexed into the correct output class.}. For a function $f(x)$, the IG attribution for the $i$th feature is defined as:
\begin{equation}
    \phi_i (x) = (x_i - x'_i) \times \int_{\alpha=0}^{1} \frac{ \partial f(x' + \alpha(x-x')) }{\partial x_i} d\alpha,
\end{equation}

where $x$ is the sample we would like to explain and $x'$ is a baseline value. Although $f$ is often a neural network, the only requirement in order to compute attribution values is that $f$ be differentiable along the path from $x'$ to $x$. Our key insight is that the IG value for a differentiable model $f : \mathbb{R}^d \mapsto \mathbb{R}$ is \emph{itself} a differentiable function $\phi_i : \mathbb{R}^d \mapsto \mathbb{R}$. This means that we can apply IG to itself in order to explain how much feature $j$ impacted the importance of feature $i$: 

\begin{equation}
    % \Gamma_{i, j}(x) = (x_j - x'_j) \times \int_{\beta=0}^{1} \frac{\partial \phi_i (x' + \beta (x - x'))}{\partial x_j} d\beta
    \Gamma_{i, j}(x) = \phi_j(\phi_i(x))
\end{equation}

For $i \neq j$, we can derive that:
\begin{equation}
    \Gamma_{i, j}(x) = 
    (x_i - x'_i)(x_j - x'_j) \times \int_{\beta=0}^{1} \int_{\alpha=0}^{1} \alpha \beta \frac{\partial^2 f(x' + \alpha \beta (x - x'))}{\partial x_i \partial x_j} d\alpha d\beta
\end{equation}

In the case of $i = j$, the formula $\Gamma_{i, i}(x)$ has an additional first-order term. We interpret $\Gamma_{i, j}(x)$ as the explanation of the importance of feature $i$ in terms of the input value of feature $j$. In practice, choosing a single baseline value can be challenging \citep{kindermans2019reliability, sundararajan2019many, sturmfels2020visualizing}. In such cases, we average over multiple baselines from the training set, as in \citet{erion2019learning}. We leave the derivation of the full formula, the baselines we choose and how we approximate the integral in practice to the appendix. 

\subsection{Fundamental Axioms for Interaction Values}

\citet{sundararajan2017axiomatic} showed that, among other theoretical properties, IG satisfies the completeness axiom, which states: $\sum_i \phi_i(x) = f(x) - f(x')$. Although we leave the derivation to the appendix, we can show the following two equalities, which are immediate consequences of completeness:

\begin{gather}
    \label{eq:interaction}
    \sum_i \sum_j \Gamma_{i, j}(x) = f(x) - f(x') \\
    \label{eq:main_effect}
    \Gamma_{i, i}(x) = \phi_i(x) - \sum_{j \neq i} \Gamma_{i, j} (x)
\end{gather}

We call equation (\ref{eq:interaction}) the \textit{interaction completeness} axiom: the sum of the $\Gamma_{i, j} (x)$ terms adds up to the difference between the output of $f$ at $x$ and at the baseline $x'$. This axiom lends itself to another natural interpretation of $\Gamma_{i, j} (x)$: as the interaction between features $i$ and $j$. That is, it represents the contribution that the pair of features $i$ and $j$ together add to the output $f(x) - f(x')$. Satisfying interaction completeness is important because it demonstrates a relationship between model output and interaction values. Without this axiom, it is unclear how to interpret the scale of interactions.

Equation (\ref{eq:main_effect}) provides a way to interpret the self-interaction term $\Gamma_{i, i}(x)$: it is the \textit{main effect} of feature $i$ after interactions with all other features have been subtracted away. We note that equation (\ref{eq:main_effect}) also implies the following, intuitive property about the main effect: if $\Gamma_{i, j} = 0$ for all $j \neq i$, or in the degenerate case where $i$ is the only feature, we have $\Gamma_{i, i} = \phi_i(x)$. We call this the  \textit{self completeness} axiom. Satisfying self-completeness is important because it provides a guarantee that the main effect of feature $i$ equals its feature attribution value if that feature interacts with no other features.

IG satisfies several other common-sense axioms such as sensitivity and linearity. We discuss generalizations of these axioms to interaction values (\textit{interaction sensitivity} and \textit{interaction linearity}) and demonstrate that Integrated Hessians satisfies said axioms in the appendix. We do observe one additional property here, which we call \textit{interaction symmetry}: that for any $i, j$, we have $\Gamma_{i, j} (x) = \Gamma_{j, i} (x)$. It is straightforward to show that existing neural networks and their activation functions have continuous second partial derivatives, which implies that Integrated Hessians satisfies interaction symmetry.\footnote{We discuss the special case of the ReLU activation function in Section \ref{sec:smoothing_relu}.} 

\section{Smoothing ReLU Networks}
\label{sec:smoothing_relu}

One major limitation that has not been discussed in previous approaches to interaction detection in neural networks is related to the fact that many popular neural network architectures use the ReLU activation function, $\textrm{ReLU}(x) = \textrm{max}(0, x)$. Neural networks that use ReLU are piecewise linear and have second partial derivatives equal to zero in all places. Previous second-order approaches (based on the Hessian) fail to detect any interaction in ReLU-based networks.

Fortunately, the ReLU activation function has a smooth approximation -- the SoftPlus function: $ \textrm{SoftPlus}_{\beta} (x) = \frac{1}{\beta} \log (1 + e^{-\beta x} )$. SoftPlus more closely approximates the ReLU function as $\beta$ increases, and has well-defined higher-order derivatives. Furthermore, \citet{dombrowski2019explanations} have proved that both the outputs and first-order feature attributions of a model are minimally perturbed when ReLU activations are replaced by SoftPlus activations in a trained network. Therefore, we can apply Integrated Hessians on a network with ReLU activations by first replacing ReLU with SoftPlus. We note that no re-training is necessary for this approach.

In addition to being twice differentiable and allowing us to calculate interaction values in ReLU networks, replacing ReLU with SoftPlus leads to other desirable behavior for calculating interaction values. We show that smoothing a neural network (decreasing the value of $\beta$ in the SoftPlus activation function) lets us accurately approximate the Integrated Hessians value with fewer gradient calls.

\begin{theorem} 
\label{thm1}
For a one-layer neural network with softplus\textsubscript{$\beta$} non-linearity, $f_{\beta}(x) =$ softplus\textsubscript{$\beta$}$(w^Tx)$, and $d$ input features, we can bound the number of interpolation points $k$ needed to approximate the Integrated Hessians to a given error tolerance $\epsilon$ by $k \leq \mathcal{O}(\frac{d \beta^2}{\epsilon})$.
\end{theorem}

Proof of \autoref{thm1} is contained in the appendix. In addition to the proof for the single-layer case, in the appendix we also show empirical results that many-layered neural networks display the same property. Finally, we demonstrate the intuition behind these results. As we replace ReLU with SoftPlus, the decision surface of the network is smoothed - see \autoref{fig:grad_path} and \citet{dombrowski2019explanations}. We can see that the gradients tend to all have more similar direction along the path from reference to foreground sample once the network has been smoothed with SoftPlus replacement.

\begin{figure}
    \centering
    \includegraphics[width=0.5\columnwidth]{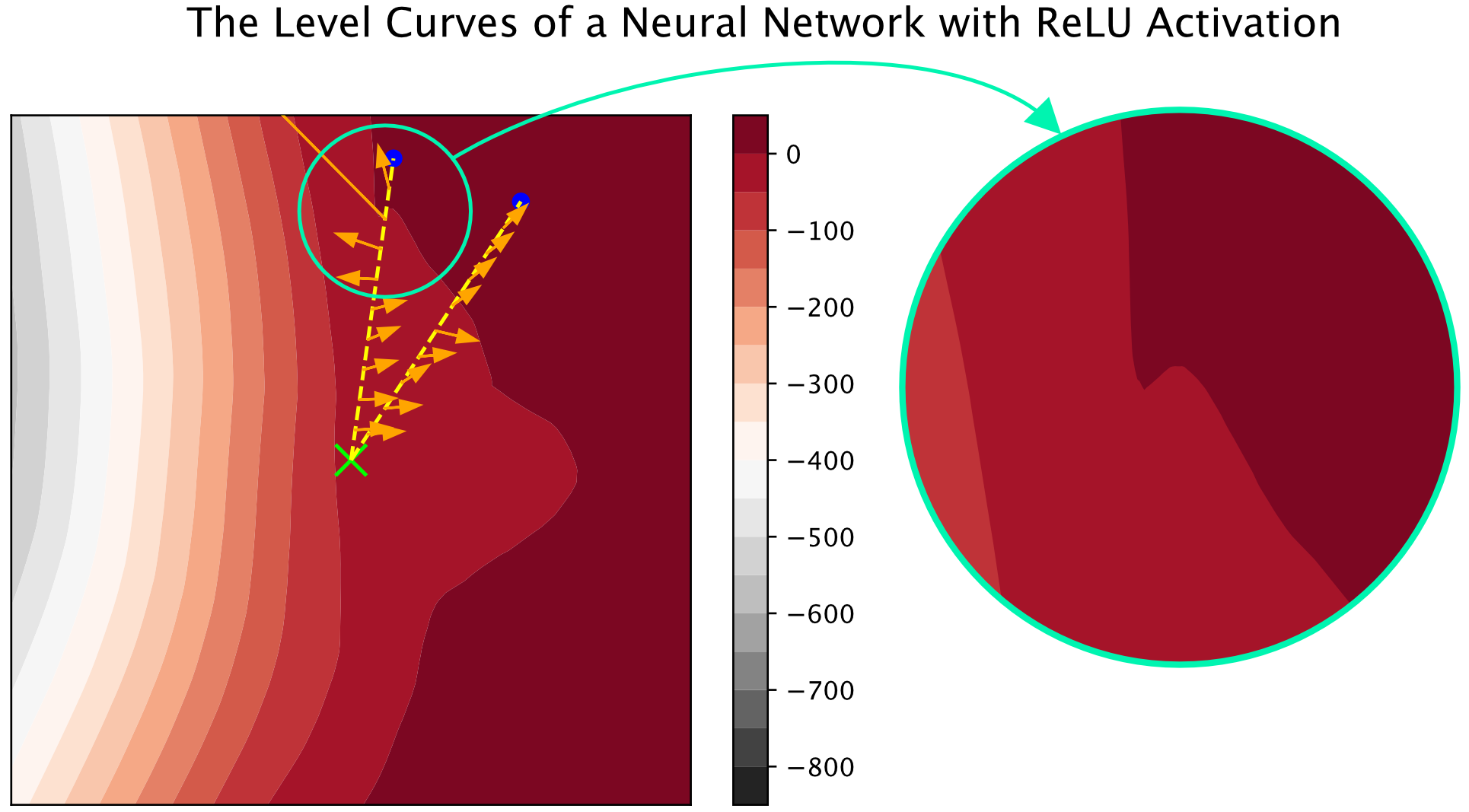}
    \includegraphics[width=0.3\columnwidth]{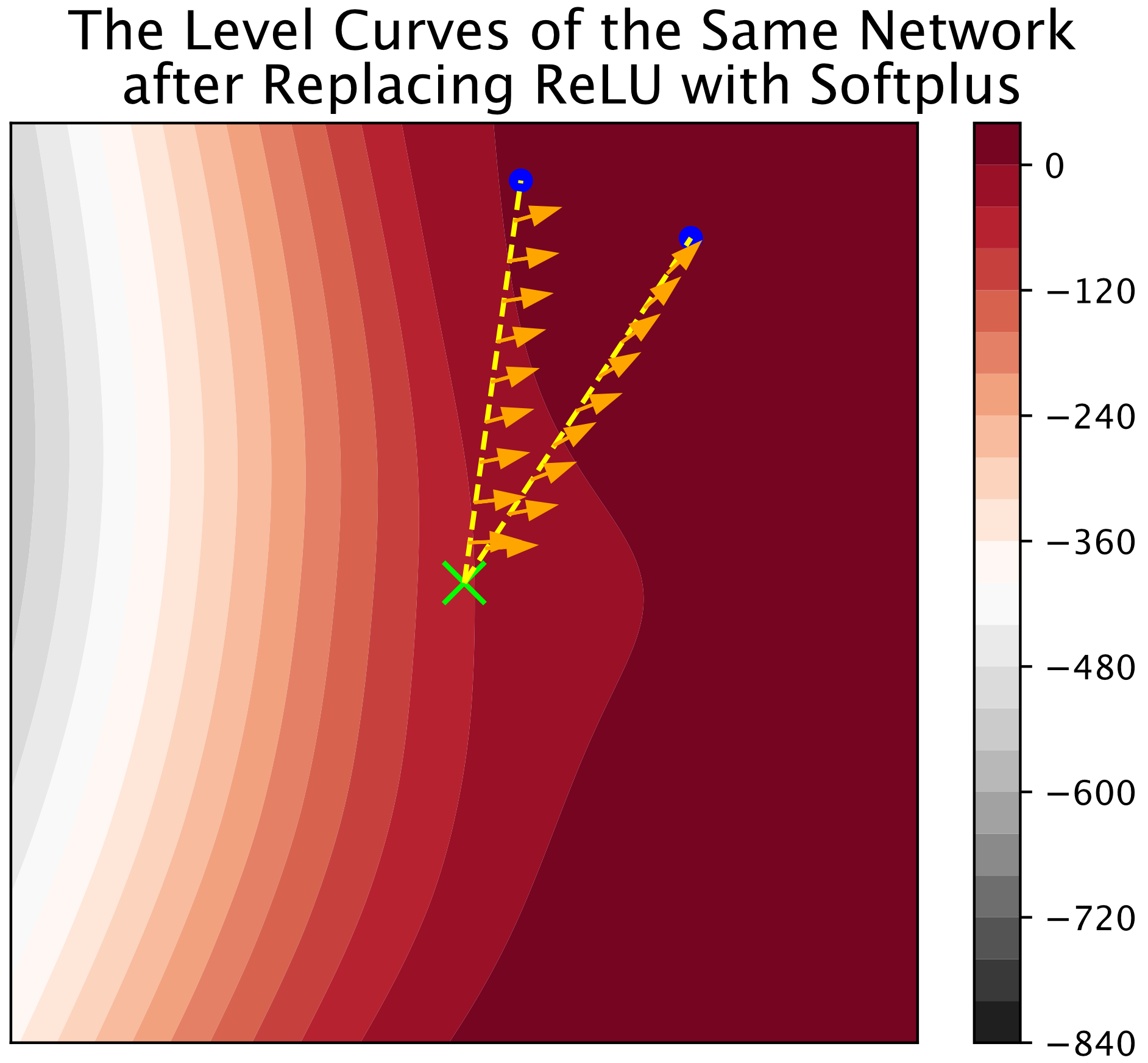}
    \caption{ Replacing ReLU activations with $\textrm{SoftPlus}_{\beta}$ activations with $\beta = 10$ smooths the decision surface of a neural network: gradients tend to be more homogeneous along the integration path. Orange arrows show the gradient vectors at each point along the path from the reference (green x) to the input (blue dots). ReLUs can cause small bends in the output space with aberrant gradients. }
    \label{fig:grad_path}
\end{figure}

\section{Explanation of XOR function}
\label{sec:SimulatedXOR}

\begin{figure}
    \centering
    \includegraphics[width=\columnwidth]{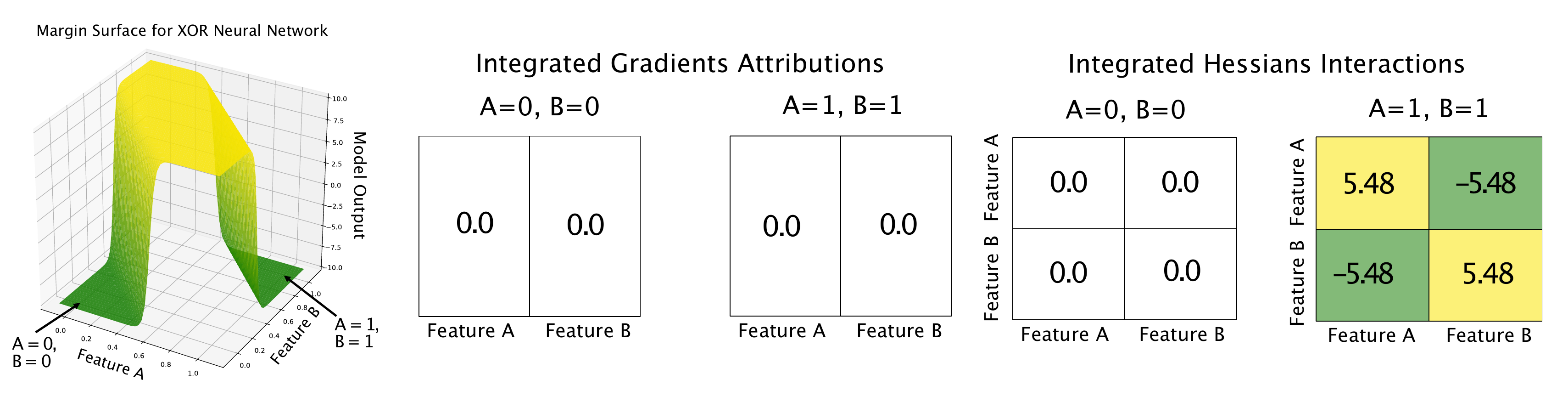}
    \caption{(Left) The margin surface for a neural network representing an XOR function. (Middle) The Integrated Gradients feature attributions for two samples, the first where both features are turned off, and the second where both features are turned on. Both get identical Integrated Gradients attributions. (Right) The Integrated Hessians feature interactions for the same two samples. We see that the Integrated Hessians values, but not the Integrated Gradients values, differentiate the two samples.}
    \label{fig:XOR}
\end{figure}

To understand why feature interactions can be more informative than feature attributions, we take the case where we have two binary features and a neural network representing an XOR function. This network has a large output when either is on alone, but a low magnitude output when both features are either on or off (\autoref{fig:XOR}, left).

When we explain the network using Integrated Gradients with the zeros baseline, we see that samples where both features are on and samples where both features are off get identical attributions (see \autoref{fig:XOR}, middle). Integrated Hessians, however, differentiates between these two samples by identifying the negative interaction that occurs between the two features when they are both on (\autoref{fig:XOR}, right). Therefore the interactions are usefully able to distinguish between (0,0), which has an output of 0 because it is identical to the baseline, and (1,1), which has an output of 0 because both features are on, which on their own should increase the model output, but in interaction with each other cancel out the positive effects and drive the model's output back to the baseline.

This example also illustrates a problem with methods like \citet{cui2019recovering}, which use the input Hessian without integrating over a path. In \autoref{fig:XOR}, we see that the function is saturated at all points on the data manifold, meaning the elements of the Hessian will be 0 for all samples. In contrast, by integrating between the baseline and the samples, Integrated Hessians is capable of correctly detecting the negative interaction between the two features.

\section{Empirical Evaluation}

In this section, we empirically evaluate our method against existing methods, inspired by recent literature on quantitatively evaluating feature attribution methods \citep{adebayo2018sanity, kindermans2019reliability, hooker2019benchmark, yeh2019fidelity, lin2019explaining}. We compare against four existing methods: the Shapley Interaction Index \cite{grabisch1999axiomatic}, Generalized Contextual Decomposition \cite{singh2018hierarchical}, Neural Interaction Detection \cite{tsang2017detecting}, and using the Hessian at the input sample \cite{cui2019recovering}.

\subsection{Computation time}

First we compare the computation time of each method. We explain interactions in a 5 layer neural network with softplus activation and 128 units per layer. We run each method on models with 5, 50 and 500 input features, and run each method on 1000 samples. Because computing the Shapley Interaction Index is NP-hard in the general case \cite{elkind2009computational}, we instead compare against a monte-carlo estimation of the Shapley Interaction Index with 200 samples, analogous to how the Shapley Value is estimated in \citet{kononenko2010efficient}. Even using a small number of samples, this was unable to be run to completion for the 500 feature case and would take an estimated 1000 hours to complete. For each method, we compute all pairwise interactions: $d^2$ interactions for $d$ features. The results are displayed in Figure \ref{fig:computation_time} on the left. \\

While our method is slower than existing methods for a small number of features, as the number of features grows, our method becomes \textit{more tractable} than existing methods. This is because other methods require at least $O(d^2)$ separate forward passes of the model to compute the interactions, which is non-trivial to parallelize. However, back-propagating the Hessian through the model is easy to do in parallel on a GPU: such functionality already exists in modern deep learning frameworks. 

\begin{figure}
    \centering
    \includegraphics[width=0.48\columnwidth]{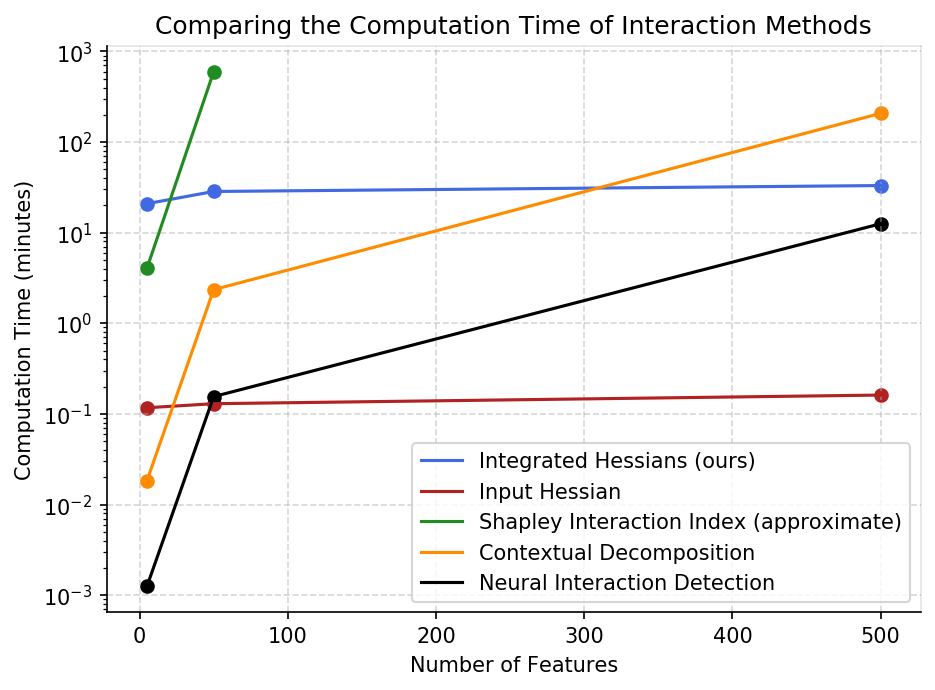}
    \includegraphics[width=0.5\columnwidth]{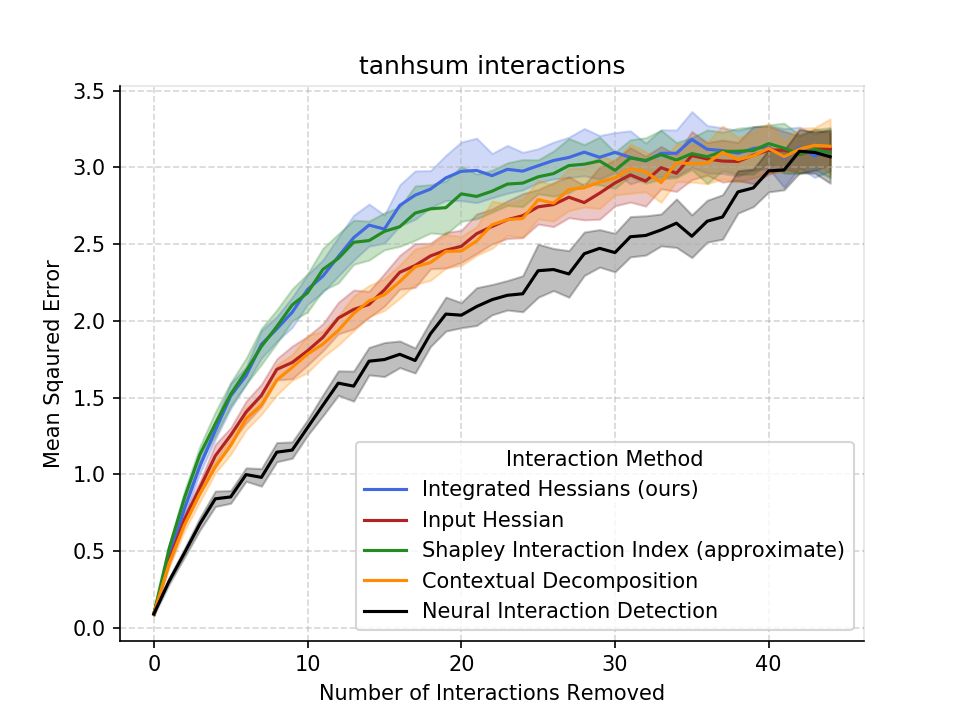}
    \caption{\textit{Left}: The time each methods takes to compute all pairwise interactions on 1000 samples with $d$ features, as a function of $d$. Existing methods are fast when $d$ is small, but scale poorly compared to our proposed method or using the input Hessian directly. \textit{Right}: Benchmarking our methods against others using a modified version of Remove and Retrain \citep{hooker2019benchmark} on simulated interactions. Our method identifies the most important interactions more quickly than all existing methods except the approximate Shapley Interaction Index, which is much more computationally expensive.}
    \label{fig:computation_time}
\end{figure}

\subsection{Quantitative comparison of interaction detection with other methods}
% Here we compare the ability of each method to identify known interactions in simulated data. 

% \subsubsection{Remove and retrain}
To compare each method, we use the Remove and Retrain benchmark introduced in \citet{hooker2019benchmark}. The benchmark compares feature attributions by progressively ablating the most important features in each sample - ranked according to each attribution method - then retraining the model on the ablated data and measuring the performance drop. The attribution method that incurs the fastest drop in performance is the one that most quickly identifies the most predictive features in the data. 

We generate a simulated regression task with 10 features where each feature is drawn independently from $\mathcal{N}(0, 1)$. The label is an additive sum of 20 interactions with random coefficients normalized to sum to 1, drawn without replacement from all possible pairs of features. A 3-layer neural network's predictions achieve over 0.99 correlation with the true label. We leave further details about experimental setup, as well as how we ``remove'' an interaction between two features to the appendix.

We compare each of the five interaction methods using Remove and Retrain on five different interaction types: $g_{\textrm{tanhsum}}(x_i, x_j) = \tanh(x_i + x_j)$, $g_{\textrm{cossum}}(x_i, x_j) = \cos(x_i + x_j)$, $g_{\textrm{multiply}}(x_i, x_j) = x_i * x_j$, $g_{\textrm{max}}(x_i, x_j) = \max(x_i, x_j)$ and $g_{\textrm{min}}(x_i, x_j) = \min(x_i, x_j)$. The results for $g_{\textrm{tanhsum}}$ are displayed in Figure \ref{fig:computation_time} on the right. Our method most quickly identifies the largest interactions in the data, as demonstrated by the fastest increase in error. We include the remaining results in the appendix, but note here that our method outperforms all existing methods on all interaction types.

In addition to the Remove and Retrain benchmark, we also directly measured the rank correlation between different interaction methods and the true interactions in two synthetic datasets. The label for the first dataset was the sum of all of the multiplicative pairwise interactions, while the label for the second was the sum of pairwise min/max functions. The features for both datasets were five independent Gaussian features. We found that the interactions attained using Integrated Hessians consistently outperformed the other methods, and include details in the appendix. Finally, we tested our approach using the ``Sanity Checks'' proposed in \cite{adebayo2018sanity} to ensure that our interaction attributions were sensitive to network and data randomization, and found that our method passed both sanity checks (see appendix for more details).

\section{Applications of Integrated Hessians}

% In this section, we outline two different use cases of Integrated Hessians on real-world data.

\subsection{NLP}
\label{sec:nlp}

In the past decade, neural networks have been the go-to model for language tasks, from convolutional \citep{kim2014convolutional} to recurrent \citep{sundermeyer2012lstm}. More recently large, pre-trained transformer architectures \citep{peters2018deep, devlin2018bert} have achieved state of the art performance on a wide variety of tasks. Some previous work has suggested looking at the internal weights of the attention mechanisms in attention-based models \citep{ghaeini2018interpreting, lee2017interactive, lin2017structured, wang2016attention}. However, more recent work has suggested that looking at attention weights may not be a reliable way to interpret models with attention layers \citep{serrano2019attention, jain2019attention, brunner2019identifiability}. To overcome this, feature attributions have been applied to text classification models to understand which words most impacted the classification \citep{liu2019incorporating, lai2019many}. However, these methods do not explain how words interact with their surrounding context.

We download pre-trained weights for DistilBERT \citep{sanh2019distilbert} from the HuggingFace Transformers library \citep{Wolf2019HuggingFacesTS}. We fine-tune the model on the Stanford Sentiment Treebank dataset \citep{socher2013recursive} in which the task is to predict whether or not a movie review has positive or negative sentiment. After 3 epochs of fine-tuning, DistilBERT achieves a validation accuracy of 0.9071 (0.9054 TPR / 0.9089 TNR).\footnote{This performance does not represent state of the art, nor is sentiment analysis representative of the full complexity of existing language tasks. However, our focus in this paper is on explanation and this task is easy to fine-tune without needing to extensively search over hyper-parameters.} We leave further fine-tuning details to the appendix.

In Figure \ref{fig:painfully_funny_and_saturation}, we show interactions generated by Integrated Hessians and attributions generated by Integrated Gradients on an example drawn from the validation set. The figure demonstrates that DistilBERT has learned intuitive interactions that would not be revealed from feature attributions alone. For example, a word like ``painfully," which might have a negative connotation on its own, has a large positive interaction with the word ``funny" in the phrase ``painfully funny." In Figure \ref{fig:not_bad_comp}, we demonstrate how interactions can help us understand one reason why a fine-tuned DistilBERT model outperforms a simpler model: a convolutional neural network (CNN) that gets an accuracy of 0.82 on the validation set. DistilBERT picks up on positive interactions between negation words (``not'') and negative adjectives (``bad'') that a CNN fails to fully capture. Finally, in Figure \ref{fig:painfully_funny_and_saturation}, we use interaction values to reveal saturation effects: many negative adjectives describing the same noun interact positively. Although this may seem counter-intuitive at first, it reflects the structure of language. If a phrase has only one negative adjective, it stands out as the word that makes the phrase negative. At some point, however, describing a noun with more and more negative adjectives makes any individual negative adjective less important towards classifying that phrase as negative.

\begin{figure}
    \centering
    \includegraphics[width=0.99\columnwidth]{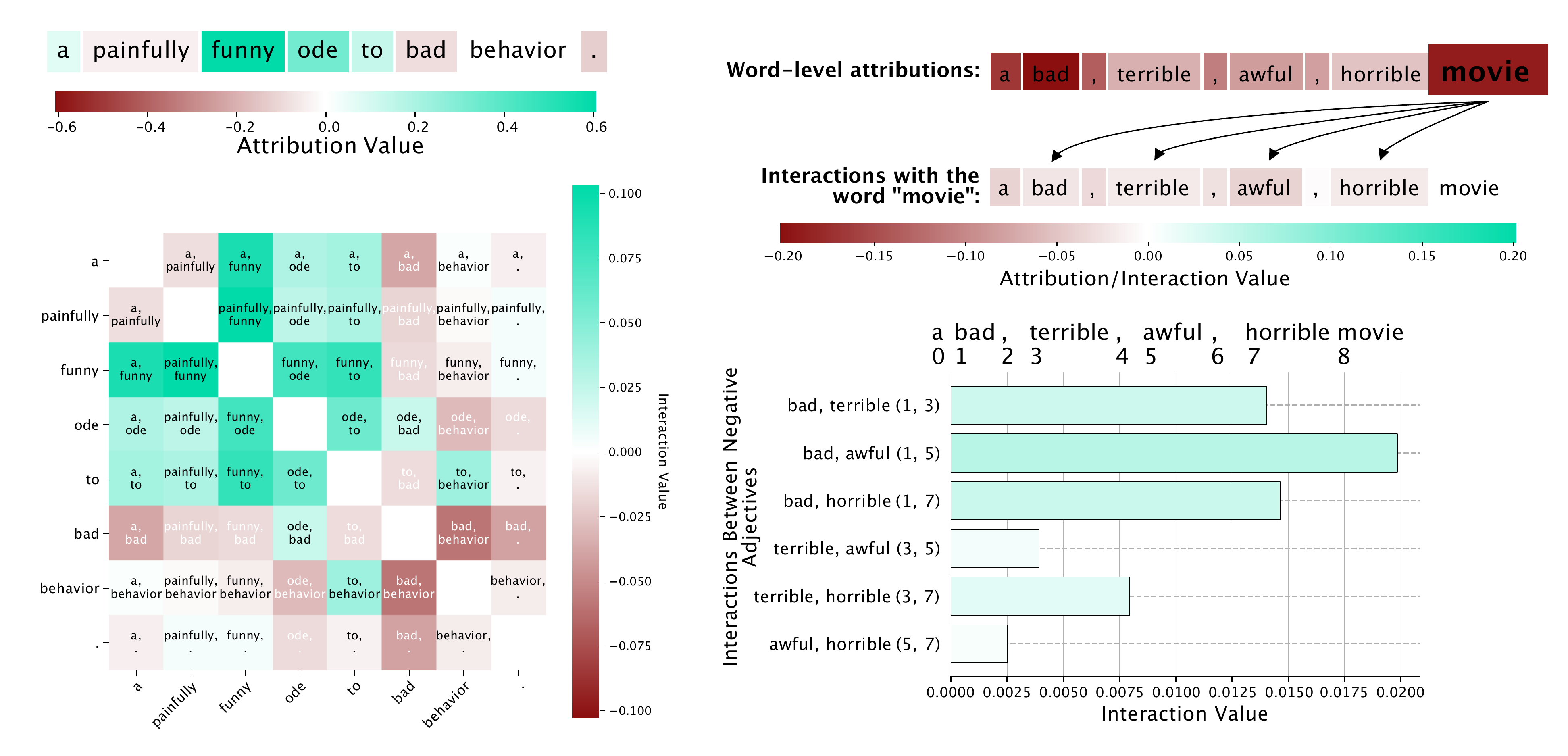}
    \caption{\textit{Left}: Interactions in text reveal learned patterns such as the phrase "painfully funny" having positive interaction despite the word "painfully" having negative attribution. These interactions are not evident from attributions alone. \textit{Right}: Interactions help us reveal an unintuitive pattern in language models: saturation. Although the word "movie" interacts negatively with all negative modifying adjectives, those negative adjectives themselves all interact positively. The more negative adjectives are in the sentence, the less each individual negative adjective matters towards the overall classification of the sentence.}
    \label{fig:painfully_funny_and_saturation}
\end{figure}

% \begin{figure}
%     \centering
%     \includegraphics[width=0.50\columnwidth]{main_figs/nlp/test_set/painfully_funny_text.pdf}
%     \includegraphics[width=0.60\columnwidth]{main_figs/nlp/test_set/painfully_funny_matrix.pdf}
%     \caption{ Interactions in text reveal learned patterns such as the phrase "painfully funny" having positive interaction despite the word "painfully" having negative attribution. These interactions are not evident from attributions alone. }
%     \label{fig:painfully_funny_matrix}
% \end{figure}

% \begin{figure}
%     \centering
%     \includegraphics[width=0.85\columnwidth]{main_figs/nlp/custom/explaining_the_word_movie.pdf}

%     \vspace{0.35cm}
    
%     \includegraphics[width=0.85\columnwidth]{main_figs/nlp/custom/saturation_effects_in_language.pdf}
%     \caption{ Interactions help us reveal an unintuitive pattern in language models: saturation. Although the word "movie" interacts negatively with all negative modifying adjectives, those negative adjectives themselves all interact positively. The more negative adjectives are in the sentence, the less each individual negative adjective matters towards the overall classification of the sentence. }
%     \label{fig:saturation_effects}
% \end{figure}

\subsection{Drug combination response prediction}
\label{sec:drug}

In the domain of anti-cancer drug combination response prediction, plotting Integrated Hessians helps us to glean biological insights into the process we are modeling. We consider one of the largest publicly-available datasets measuring drug combination response in acute myeloid leukemia \citep{tyner2018functional}. Each one of 12,362 samples consists of the measured response of a 2-drug pair tested in the cancer cells of a patient. The 1,235 input features are split between features describing the drug combinations and features describing the cancerous cells, which we modeled using the neural architectures described in \citet{hao2018pasnet} and \citet{10.1093/bioinformatics/btx806}.

% Following \citet{hao2018pasnet}, we first learn an embedding of biological pathways from individual gene expression levels in order to have a more interpretable model. Following \citet{10.1093/bioinformatics/btx806}, the drug features and pathway features are then input into a simple feed-forward neural network (see appendix for more details).

According to the first-order explanations, the presence or absence of the drug Venetoclax in the drug combination is the most important feature. We can also easily see that first-order explanations are inadequate in this case -- while the presence of Venetoclax is generally predictive of a more responsive drug combination, the amount of positive response to Venetoclax is predicted to vary across samples (see \autoref{fig:drug_fig}, top left).

Integrated Hessians gives us the insight that some of this variability can be attributed to the drug Venetoclax is combined with. We can see that the model has learned a strong negative interaction between Venetoclax and Artemisinin (see \autoref{fig:drug_fig}, middle), which we can confirm matches the ground truth ascertained from additional external data ($p = 2.31 \times 10^{-4}$, see appendix for details of calculation). Finally, we can gain insight into the variability in the \emph{interaction} values between Venetoclax and Artemisinin by plotting them against the expression level of a pathway containing cancer genes (see \autoref{fig:drug_fig}, bottom). We see that patients with higher expression of this pathway tend to have a more negative interaction (sub-additive response) than patients with lower expression of this pathway. Integrated Hessians helps us understand the interactions between drugs in our model, as well as what genetic factors influence this interaction.
% Biological interactions are known to occur between anti-cancer drugs, and are an area of great clinical interest to to their potential therapeutic effects. Using additional data not directly available to the model, we can determine the ground truth as to which patients actually had positive and negative biological interactions between Venetoclax and Artemisinin (see appendix for details of calculation). We see that the Integrated Hessians interaction values are significantly more negative in the group with real biological negative interactions ($p = 2.31 \times 10^{-4}$).

\begin{figure}
    \centering
    \includegraphics[width=0.85\columnwidth]{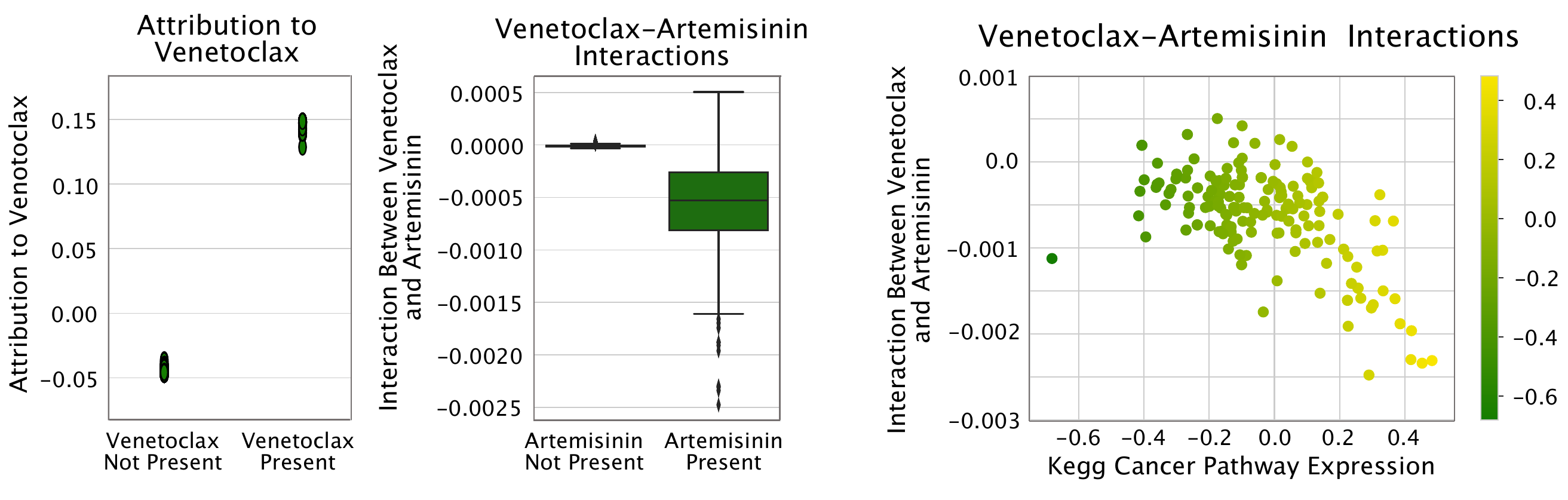}
    \caption{Left: Integrated Gradients values for Venetoclax. Middle: Venetoclax interactions with Artemisinin across all samples. Right: Venetoclax and Artemisinin interaction is driven by expression of genes in cancer samples.}
    \label{fig:drug_fig}
\end{figure}

% \section{Additional Related Work}

% Although we covered the bulk of the related work in Section \ref{sec:Intro}, we mention some additional related work here. Detecting interactions - when two or more variables have a combined effect not equal to their additive effect - has a long history in statistics \citep{southwood1978substantive}, economics \citep{balli2013interaction} and game theory \citep{grabisch1999axiomatic}. In the context of machine learning, many methods have been proposed to \textit{learn} interactions from data, e.g. using additive models \citep{coull2001simple, lou2013accurate}, group-lasso \citep{lim2015learning}, or trees \citep{sorokina2008detecting}. We view these approaches as tangential to our work: they aim to learn new models to detect specific, pairwise interactions; we propose a method to explain interactions in deep networks that have already been trained. This is especially important in domains where neural networks achieve state of the art performance, like NLP.

% We are unaware of any work to explain interactions in neural networks other than those works mentioned in Section \ref{sec:Intro} \citep{cui2019recovering, tsang2017detecting, greenside2018discovering}. Parallel to our work, \citet{lundberg2020local} propose an extension of feature attribution methods to feature interactions for tree-based models which satisfies similar properties to the ones we propose. 

\section{Conclusion}

In this work we propose a novel method to explain feature interactions in neural networks. The interaction values we propose have two natural interpretations: (1) as the combined effect of two features to the output of a model, and (2) as the explanation of one feature's importance in terms of another. Our method provably satisfies common-sense axioms that previous methods do not - and outperforms previous methods in practice at identifying known interactions on simulated data.

Additionally, we demonstrate how to glean interactions from neural networks trained with a ReLU activation function which has no second-derivative. In accordance with recent work, we show why replacing the ReLU activation function with the softplus activation function at explanation time is both intuitive and efficient.

Finally, we perform several experiments to reveal the utility of our method, from understanding performance gaps between model classes to discovering patterns a model has learned on high-dimensional data. We conclude that although feature attribution methods provide valuable insight into model behavior, such methods by no means end the discussion on interpretability. Rather, they encourage further work in deeper understanding model behavior.

\newpage
\section*{Broader Impact}

We have shown that our proposed technique, Integrated Hessians, is capable of offering significant insight into the behavior of neural networks. As such, it has the potential to significantly impact a wide variety of stakeholders -- ranging from industry practitioners attempting to debug neural networks for deployment to individual consumers looking to understand why a particular algorithmic decision impacted them. While our theoretical and empirical results have conclusively demonstrated the technical soundness of our method, it is instructive to consider some of the specific cases where a method like ours could be used in order to understand the possible social impacts.

We first consider the case of a scientist using our method to gain mechanistic insight into a biological process by modeling their data with a neural network. Since this data is typically high-dimensional with correlated features and large amounts of noise and batch effects, it is possible that the model depends more on spurious correlates than actual causally-meaningful features. If the neural network learns confounded relationships, the feature attributions and feature interactions our method reveals for this network will reflect that confounding. While this is not a technical weakness of our method, it is a very real possibility that it could have negative impact in practical deployment if scientists misunderstand how the feature interaction values relate to the data generating process. The feature interactions are causal \textit{with respect to the model}, but not necessarily \textit{with respect to the true data generating process}. This potential misunderstanding highlights the importance of educating users to this risk, which we will aim to do with example notebooks and tutorials in the open source repository for our code.

We next consider the case of a loan applicant who wants to understand why their loan application was accepted or denied. In this scenario, both the applicant and the lender stand to benefit from our approach. The lender is able to make their most accurate ``black box'' models compliant with the GDPR's ``right to an explanation,'' while the loan applicant is able to gain detailed insight into the reasons their loans was accepted or denied. In this application, there is a high degree of financial incentive for malicious use. For example, one could imagine a lender who wanted to hide the fact that their model unfairly depends on protected characteristics like race or gender. It has previously been shown that first-order interpretability methods (like LIME, SHAP, and Integrated Gradients) are subject to adversarial manipulation, and it is likely that our interaction method is also vulnerable in this way \citep{ghorbani2019interpretation, slack2020fooling}. Therefore, there is a necessity to either make sure the explanations are deployed by trustworthy individuals (e.g. impartial auditors rather than stakeholders), or to deploy feature explanation and interaction tools in conjunction with other tools that can modify underlying models to ensure robustness to this sort of adversarial attack \citep{chen2019robust}.

In all of the discussed cases, we see that our technique has the potential to benefit a wide variety of stakeholders if used properly. We can also see the importance of helping potential users understand how to properly and safely use our method.

\newpage
\bibliography{refs}

\newpage
\appendix
\onecolumn
\section*{Appendix}

\section{Deriving Interaction Values}

In this section, we expand the derivation of the Integrated Hessians formula, further discuss relevant axioms that interaction values should satisfy, and how we compute Integrated Hessians in practice. 

\subsection{The Integrated Hessians Formula}
Here we derive the formula for Integrated Hessians from it's definition: $\Gamma_{i, j}(x) = \phi_j(\phi_i(x))$. We start by expanding out $\phi_j$ using the definition of Integrated Gradients:

\begin{align}
    \Gamma_{i, j} (x) := (x_j - x_j') \times \int_{\beta=0}^1 \frac{\partial \phi_i (x' + \beta (x - x'))}{\partial x_j} d\beta
\end{align}

We consider the function $\frac{\partial \phi_i}{\partial x_j} (x)$, and we first assume that $i \neq j$
\begin{align}
    \frac{\partial \phi_i}{\partial x_j} (x) &= \\
    (x_i - x_i') \times \frac{\partial}{\partial x_j} \Bigg( \int_{\alpha=0}^1 \frac{\partial f(x' + \alpha (x - x'))}{\partial x_i} d\alpha \Bigg) &= \\
    (x_i - x_i') \times \int_{\alpha=0}^1 \frac{\partial}{\partial x_j} \Bigg(  \frac{\partial f(x' + \alpha (x - x'))}{\partial x_i} \Bigg) d\alpha &= \\
    (x_i - x_i') \times \int_{\alpha=0}^1 \alpha \frac{\partial^2 f(x' + \alpha (x - x'))}{\partial x_i \partial x_j} d\alpha
    \label{eq:partial}
\end{align}

where we've assume that the function $f$ satisfies the conditions for the Leibniz Integral Rule (so that integration and differentiation are interchangeable). These conditions require that the derivative of $f$, $\frac{\partial f}{\partial x_i}$ and its second derivative function $\frac{\partial^2 f}{\partial x_i \partial x_j}$ are continuous over $x$ in the integration region, and that the bounds of integration are constant with respect to $x$. It's easy to see that the bounds of integration are constant with respect to $x$. It is also straightforward to see that common neural network activation functions - for example, $\textrm{sigmoid}(x) = \frac{1}{1+e^{-x}}$, $\textrm{softplus}_\beta(x) = \frac{1}{\beta}\log(1 + e^{-\beta x})$, or $\textrm{gelu}(x) = x \Phi(x)$ where $\Phi(x)$ is the cumulative distribution function of the normal distribution - have continuous first and second partial derivatives, which implies compositions of these functions have continuous first and second partial derivatives as well. Although this is not the case with the ReLU activation function, we discuss replacing it with softplus in the main text. \\

We can proceed by plugging equation \ref{eq:partial} into the original definition of $\Gamma_{i, j}(x)$:
\begin{align}
    \Gamma_{i, j} (x) := (x_j - x_j') \times \int_{\beta=0}^1 \frac{\partial \phi_i (x' + \beta (x - x'))}{\partial x_j} d\beta &= \\
    (x_j - x_j') \times \int_{\beta=0}^1 (x'_i - \beta (x_i - x'_i) - x'_i) \int_{\alpha=0}^1 \alpha \frac{\partial^2 f(x' + \alpha (x' - \beta (x - x') - x'))}{\partial x_i \partial x_j} d\alpha d\beta &=\\
    (x_j - x_j') (x_i - x_i') \int_{\beta = 0}^{1} \int_{\alpha = 0}^{1} \alpha \beta \frac{\partial^2 f(x' + \alpha \beta (x - x'))}{\partial x_i \partial x_j} d\alpha d\beta
\end{align}

where all we've done is re-arrange terms. \\

Deriving $\Gamma_{i, i} (x)$ proceeds similarly:

\begin{align}
    \frac{\partial \phi_i}{\partial x_i} (x) &= \\
    \frac{\partial}{\partial x_i} \Bigg((x_i - x_i')\Bigg) \times  \int_{\alpha=0}^1 \frac{\partial f(x' + \alpha (x - x'))}{\partial x_i} d\alpha + (x_i - x_i') \times \frac{\partial}{\partial x_i} \Bigg( \int_{\alpha=0}^1 \frac{\partial f(x' + \alpha (x - x'))}{\partial x_i} d\alpha \Bigg) &= \\
    \int_{\alpha=0}^1 \frac{\partial f(x' + \alpha (x - x'))}{\partial x_i} d\alpha + (x_i - x_i') \times \int_{\alpha=0}^1 \alpha \frac{\partial^2 f(x' + \alpha (x - x'))}{\partial x_i \partial x_j} d\alpha
\end{align}

using the chain rule. After similar re-arrangement, we can arrive at:
\begin{align}
    \Gamma_{i, i}(x) = (x_i - x_i') \int_{\beta = 0}^{1} \int_{\alpha = 0}^{1} \frac{\partial f(x' + \alpha \beta (x - x'))}{\partial x_i} d\alpha d\beta + (x_i - x_i')^2 \times \int_{\beta = 0}^{1} \int_{\alpha = 0}^{1} \alpha \beta \frac{\partial^2 f(x' + \alpha \beta (x - x'))}{\partial x_i \partial x_j} d\alpha d\beta
\end{align}

\subsection{Baselines and Expected Hessians}

Several existing feature attribution methods have pointed out the need to explain relative to a baseline value that represents a lack of information that the explanations are relative to \citep{sundararajan2017axiomatic, shrikumar2017learning, binder2016layer, lundberg2017unified}. However, more recent work has pointed out that choosing a single baseline value to represent lack of information can be challenging in certain domains \citep{kindermans2019reliability, kapishnikov2019segment, sundararajan2018note, ancona2017towards, fong2017interpretable, sturmfels2020visualizing}. As an alternative, \citet{erion2019learning} proposed an extension of IG called Expected Gradients (EG), which samples many baseline inputs from the training set. We can therefore apply EG to itself to get Expected Hessians.
\begin{align}
    \Gamma_{i, j}^{EG} (x) &= \mathbb{E}_{\alpha \beta \sim U(0, 1) \times U(0, 1), x' \sim D}  \Bigg[  (x_i - x_i')(x_j - x_j')
    \alpha \beta \frac{\partial^2 f(x' + \alpha \beta (x - x'))}{\partial dx_i \partial dx_j} \Bigg ] \\
    \Gamma_{i, i}^{EG} (x) &= \mathbb{E}_{\alpha \beta \sim U(0, 1) \times U(0, 1), x' \sim D}  \Bigg[ (x_i - x_i') \frac{\partial f(x' + \alpha \beta (x - x'))}{\partial x_i} + (x_i - x_i')^2
    \alpha \beta \frac{\partial^2 f(x' + \alpha \beta (x - x'))}{\partial dx_i \partial dx_j} \Bigg ]
\end{align}
where the expectation is over $x' \sim D$ for an underlying data distribution $D$, $\alpha \sim U(0, 1)$ and $\beta \sim U(0, 1)$. This formulation can be useful in the case where there does not exist a single, natural baseline. The derivation for Expected Hessians follows identical steps as the derivation for Integrated Hessians while observing that the integrals can be viewed as integrating over the product of two uniform distributions $\alpha\beta \sim U(0, 1) \times U(0, 1)$. We use Integrated Hessians for all of the examples in the main text - however, the additional examples shown in the appendix use the Expected Hessians formulation. 

\subsection{Axioms Satisfied by Integrated Hessians}

\subsubsection{Self and Interaction Completeness}
The proof that Integrated Hessians satisfies the axioms \textit{interaction completeness} and \textit{self completeness} are straightforward to show, but we include the step-by-step derivation here. First, we note that $\phi_i(x') = 0$ for any $i$ because $x_i' - x_i' = 0$. Then, by completeness of Integrated Gradients, we have that:
\begin{align}
    \sum_j \Gamma_{i, j}(x) = \phi_i(x) - \phi_i(x') = \phi_i(x) \label{eq:subs}
\end{align}

Re-arrangement gives us the \textit{self completeness} axioms:
\begin{align}
    \Gamma_{i, i}(x) = \phi_i(x) \textrm{ if } \Gamma_{i, j}(x) = 0, \forall j \neq i
\end{align}

Since Integrated Gradients satisfies completeness, we have:
\begin{align}
    \sum_i \phi_i(x) = f(x) - f(x')
\end{align}

Making the appropriate substitution from equation \ref{eq:subs} shows the \textit{interaction completeness} axiom:
\begin{align}
    \sum_i \sum_j \Gamma_{i, j} (x) = f(x) - f(x')
\end{align}

\subsubsection{Sensitivity}
Integrated Gradients satisfies an axiom called sensitivity, which can be phrased as follows. Given an input $x$ and a baseline $x'$, if $x_i = x_i'$ for all $i$ except $j$ where $x_j \neq x_j'$ and if $f(x) \neq f(x')$, then $\phi_j (x) \neq 0$. Specifically, by completeness we know that $\phi_j(x) = f(x) - f(x')$. Intuitively, this is saying that if only one feature differs between the baseline and the input and changing that feature changes the output, then the amount the output changes should be equal to the importance of that feature. \\

We can extend this axiom to the interaction case by considering the case when two features differ from the baseline. We call this axiom \textit{interaction sensitivity}, and can be described as follows. If an input $x$ and a baseline $x'$ are equal everywhere except $x_i \neq x_i'$ and $x_j \neq x_j'$, and if $f(x) \neq f(x')$, then: $\Gamma_{i, i}(x) + \Gamma_{j, j}(x) + 2 \Gamma_{i, j}(x) = f(x) - f(x') \neq 0$ and $\Gamma_{\ell, k} = 0$ for all $\ell, k \neq i, j$. Intuitively, this says that if the only features that differ from the baseline are $i$ and $j$, then the difference in the output $f(x) - f(x')$ must be solely attributable to the main effects of $i$ and $j$ plus the interaction between them. This axiom holds simply by applying \textit{interaction completeness} and observing that $\Gamma_{\ell, k} (x) = 0$ if $x_\ell = x'_\ell$ or $x_k = x'_k$.\\

\subsubsection{Implementation Invariance}
The implementation invariance axiom, described in the original paper, states the following. For two models $f$ and $g$ such that $f = g$, then $\phi_i(x; f) = \phi_i(x; g)$ for all features $i$ and all points $x$ regardless of how $f$ and $g$ are implemented. Although it seems trivial, this axiom does not necessarily hold for attribution methods that use the implementation or structure of the network in order to generate attributions. Critically, this axiom also does not hold for the interaction method proposed by \citet{tsang2017detecting}, which looks at the first layer of a feed forward neural network. Two networks may represent exactly the same function but differ greatly in their first layer.\\

This axiom is trivially seen to hold for Integrated Hessians since it holds for Integrated Gradients. However, this axiom is desirable because without it, it may mean that attributions/interactions are encoding information about unimportant aspects of model structure rather than the actual decision surface of the model.

\subsubsection{Linearity}
Integrated Gradients satisfies an axiom called linearity, which can be described as follows. Given two networks $f$ and $g$, consider the output of the weighted ensemble of the two networks $a f (x) + b g(x)$. Then the attribution $\phi_i(x; af + bg)$ of the weighted ensemble equals the weighted sum of attributions $a\phi_i(x; f) + b\phi_i(x; g)$ for all features $i$ and samples $x$. This axiom is desirable because it preserves linearity within a network, and allows easy computation of attributions for network ensembles.\\

We can generalize linearity to interactions using the \textit{interaction linearity} axiom: $\Gamma_{i, j} (x; af + bg) = a \Gamma_{i, j} (x; f) + b \Gamma_{i, j} (x; g)$ for any $i, j$ and all points $x$. Given that $\Gamma_{i, j}$ is composition of linear functions $\phi_i$, $\phi_j$ in terms of the parameterized networks $f$ and $g$, it itself is a linear function of the networks and therefore Integrated Hessians satisfies \textit{interaction linearity}. 

\subsubsection{Symmetry-Preserving}
We say that two features $x_i$ and $x_j$ are \textit{symmetric} with respect to $f$ if swapping them does not change the output of $f$ anywhere. That is, $f(\cdots, x_i, \cdots, x_j, \cdots) = f(\cdots, x_j, \cdots, x_i, \cdots$). The original paper shows that Integrated Gradients is \textit{symmetry-preserving}, that is, if $x_i$ and $x_j$ are symmetric with respect to $f$, and if $x_i = x_j$ and $x_i' = x_j'$ for some input $x$ and baseline $x'$, then $\phi_i(x) = \phi_j(x)$. We can make the appropriate generalization to interaction values: if the same conditions as above hold, then $\Gamma_{k, i} (x) = \Gamma_{k, j} (x)$ for any feature $x_k$. This axiom holds since, again $\Gamma_{k, i} (x) = \phi_i (\phi_k(x))$ and $\phi_i, \phi_j$ are symmetry-preserving. This axiom is desirable because it says that if two features are functionally equivalent to a model, then they must interact the same way with respect to that model.

\subsection{Approximating Integrated Hessians in Practice}
To compute Integrated Gradients in practice, \citet{sundararajan2017axiomatic} introduce the following discrete sum approximation:

\begin{align}
    \hat{\phi}_i(x) = (x_i - x_i') \times \sum_{\ell = 1}^k \frac{\partial f(x' + \frac{\ell}{k} (x - x'))}{\partial x_i} \times \frac{1}{k}
\end{align}

where $k$ is the number of points used to approximate the integral. To compute Integrated Hessians, we introduce a similar discrete sum approximation:

\begin{align}
    \hat{\Gamma}_{i, j}(x) = (x_i - x_i')(x_j - x_j') \times \sum_{\ell = l}^k \sum_{p = 1}^m \frac{\ell}{k} \times \frac{p}{m} \times \frac{\partial f(x' + (\frac{\ell}{k} \times \frac{p}{m})  (x - x'))}{\partial x_i \partial x_j} \times \frac{1}{km}
\end{align}

Typically, it is easiest to compute this quantity when $k = m$ and the number of samples drawn is thus a perfect square - however, when a non-square number of samples is desired we can generate a number of sample points from the product distribution of two uniform distributions such that the number is the largest perfect square above the desired number of samples, and index the sorted samples appropriately to get the desired number. The above formula omits the first-order term in $\Gamma_{i, i}(x)$ but it can be computed using the same principle. \\

Expected Hessians has a similar, if slightly easier form:

\begin{align}
    \hat{\Gamma}_{i,j}^{EG}(x) = (x_i - x_i')(x_j - x_j') \sum_{\ell}^k \zeta_\ell \times \frac{\partial f(x' + \zeta_\ell  (x - x'))}{\partial x_i \partial x_j} \times \frac{1}{k}
\end{align}

where $\zeta_\ell$ is the $\ell$th sample from the product distribution of two uniform distributions. We find that in general that less than 300 samples are required for any given problem to approximately satisfy interaction completeness. For most problems, a number far less than 300 suffices (e.g. around 50) although this is model and data dependent: larger models and higher-dimensional data generally require more samples than smaller models and lower-dimensional data.

\section{Comparing Against Existing Methods}

In this section, we further discuss the relationship between our method and six existing methods:
\begin{itemize}
    \item Integrated Hessians: the method we propose in this paper.
    \item Input Hessian: simply use the hessian at the input instance. Using the input hessian at a particular instance to measure the interaction values is the natural generalization of using the gradient to measure the importance of individual features, as done by \cite{simonyan2013deep}.
    \item Contextual Decomposition (CD): this method was introduced by \citet{murdoch2018beyond} for LSTMs and extended to feed-forward and convolutional architectures by \citet{singla2019understanding}. We focus on the generalized version introduced by the latter. 
    \item Neural Interaction Detection (NID): introduced by \citet{tsang2017detecting}, this method generated interactions by inspecting the weights of the first layer of a feed-forward neural network.
    \item Shapley Interaction Index (SII): this method, introduced by \citet{grabisch1999axiomatic}, is a way to allocate credit in coalitional game theory. It can be used to explain interactions in neural networks similarly to how the Shapley value is used to explain attributions in \cite{lundberg2017unified}. 
    \item Group Expected Hessian (GEH): this method aggregates the input hessian over many samples with respect to a bayesian neural network. It was introduced in \citet{cui2019recovering}.
    \item Deep Feature Interaction Maps (DFIM): this method determines interactions by seeing how much attributions change when features are perturbed, and was introduced by \citet{greenside2018discovering}.
\end{itemize}

We first discuss the practical considerations regarding each method, and then evaluate whether or not each method satisfies the axioms we identified in the previous section.

\subsection{Practical Considerations}
In this section we describe four properties desirable properties that interaction methods may or may not satisfy:

\begin{itemize}
    \item Local: a method is local if it operates at the level of individual predictions, as opposed to methods that are global, which operate over the entire dataset. Which is better is task dependent, but local methods are often more flexible because they can be aggregated globally \citep{lundberg2020local}.
    \item Architecture Agnostic: a method is architecture agnostic if it can be applied to any neural network architecture. 
    \item Data Agnostic: a method is data agnostic if it can be applied to any type of data, no matter the structure.
    \item Higher-Order Interactions: this paper primarily discusses interactions between pairs of features. However, some existing methods are able to generate interactions between groups of features larger than 2. Those methods that do are said to be able to generate higher-order interactions.
\end{itemize}

In Table \ref{tab:practical}, we compare which methods satisfy which properties. We note that GEH and DFIM are not architecture or data agnostic, respectively. Therefore, we do not include them in empirical comparisons, as they cannot be run on feed-forward neural networks with arbitrary data. We also note that our method does not generate higher-order interactions. Although in principle one generate $k$th order interactions by recursively applying integrated gradients to itself $k$ times, we do not discuss doing so in this paper.

\begin{table*}
    \centering
    \begin{tabular}{ |c|c|c|c|c|  }
     \hline
     \textbf{Interaction Method} & Local & \makecell{Architecture \\Agnostic} & \makecell{Data \\Agnostic} & \makecell{Higher-Order\\ Interactions}\\
     \hline
     Input Hessian & \checkmark & \checkmark & \checkmark &  \\
     \hline
     Integrated Hessians (ours) & \checkmark & \checkmark & \checkmark & \\
     \hline
     CD \citep{singh2018hierarchical} & \checkmark & $\sim$* & \checkmark & \checkmark \\
     \hline
     NID \citep{tsang2017detecting} & & \checkmark & \checkmark & \checkmark \\
     \hline
     SII \citep{grabisch1999axiomatic} & \checkmark & \checkmark & \checkmark & \checkmark\\
     \hline
     GEH \citep{cui2019recovering} & \checkmark & & \checkmark & \\
     \hline
     DFIM \citep{greenside2018discovering} & \checkmark & \checkmark & & \\
     \hline
    \end{tabular}
    \caption{Comparing the practical properties of existing interaction methods. Because GEH and DFIM are not architecture or data agnostic respectively, we omit them from empirical comparisons. *Contextual Decomposition was originally introduced for LSTMs in \citet{murdoch2018beyond}, and generalized to feed-forward and convolutional architectures in \citet{singh2018hierarchical}. However, the method has yet to be adapted to other architectures such as transformers \cite{devlin2018bert}.}
    \label{tab:practical}
\end{table*}

\subsection{Theoretical Considerations}

In this section, we evaluate which methods satisfy the axioms we presented in the previous section. The results are presented in Table \ref{tab:axioms}. We note that Integrated Hessians and the Shapley Interaction Index share many theoretical properties, except that the Shapley Interaction Index trades completeness for the recursive axioms, which state that higher order interactions should be recursively defined from lower order interactions. Whether or not this is preferable to satisfying completeness seems subjective. \\

The hessian satisfies implementation invariance for the same reason that our method does: it only relies on the value of the network function and its derivatives. The hessian also satisfies linearity since the derivative is a linear operator. However, second derivatives are not symmetry preserving for general functions (this is easy to see by considering the multiplication of two variables). Contextual decomposition is symmetry preserving by definition, but fails on implementation invariance due to the way it splits the bias. 

\begin{table*}
    \centering
    \begin{tabular}{ |c|c|c|c|c|c|c|  }
     \hline
     \textbf{Interaction Method} & \makecell{Interaction \\Completeness} & \makecell{Interaction \\Sensitivity} & \makecell{Implementation \\Invariant} & \makecell{Symmetry\\ Preserving} & \makecell{Interaction \\Linearity} & \makecell{Recursive \\Axioms*} \\
     \hline
     Input Hessian &  &  &  \checkmark & & \checkmark & \\
     \hline
     Integrated Hessians (ours) & \checkmark & \checkmark & \checkmark & \checkmark & \checkmark & \\
     \hline
     CD \citep{singh2018hierarchical} & & $\sim$** & & \checkmark & & \\
     \hline
     NID \citep{tsang2017detecting} & & & & & & \\
     \hline
     SII \citep{grabisch1999axiomatic} & & \checkmark & \checkmark & \checkmark & \checkmark & \checkmark \\
     \hline
     GEH \citep{cui2019recovering} & & & \checkmark & & \checkmark & \\
     \hline
     DFIM \citep{greenside2018discovering} &  &  & $\sim$*** & $\sim$*** & $\sim$*** & \\
     \hline
    \end{tabular}
    \caption{Comparing the theoretical guarantees of existing interaction methods. We define each axiom in the appendix. *The recursive axioms are discussed in \citet{grabisch1999axiomatic} and guarantee that higher-order interactions satisfy a specific recurrence relation. **Contextual Decomposition does not satisfy the exact sensitivity axiom we present in the appendix, but it does guarantee that features equal to 0 will have zero interaction with any other feature. ***Whether or not Deep Feature Interaction Maps satisfies these properties relies on the underlying feature attribution method used: if the underlying method satisfies the properties, so does DFIM. }
    \label{tab:axioms}
\end{table*}

\subsection{Quantitative Evaluations}

In this section, we elaborate on the quantitative comparisons presented in the main text. We also mention here that we do not compare against GEH and DFIM, because they are not applicable to feed-forward neural networks with continuous data.

\subsubsection{Remove and Retrain}

The Remove and Retrain benchmark first starts with a trained network on a given dataset. It ranks the features most important for prediction on every sample in the dataset according to a given feature attribution method. Using the ranking, it iteratively ablates the most important features in each sample and then re-trains the model on the ablated data. In order to run this method, it is necessary to be able to ablate a feature, which the original paper does by mean or zero imputation \citep{hooker2019benchmark}. However, this presents a problem when trying to compare methods that explain interactions between features rather than attributions to features.\\

Unlike ablating features, it is not straightforward to ablate an interaction between two features. Ablating the both of the features in the interaction doesn't work, because it mixes main effects with interactions. Consider the function $f(x_1, x_2, x_3, x_4) = x_1 + x_2 + 0.1 * x_3 * x_4$. For $f(1.0, 1.0, 1.0, 1.0) = 2.1$, the largest and only interaction is between $x_3$ and $x_4$. However, ablating the pair $x_1, x_2$ would incur a larger performance hit because the features $x_1$ and $x_2$ have larger main effects than $x_3$ and $x_4$. \\

Instead, we opt to generate simulated data where the interactions are \textit{known}, and then ablate the interactions in the labels directly. We generate a simulated regression task with 10 features where each feature is drawn independently from $\mathcal{N}(0, 1)$. The label is an additive sum of 20 interactions with random coefficients, drawn without replacement from all possible pairs of features. That is, for some specified interaction function $g$, we generate the label $y$ as:
\[ \sum_{i = 1}^{20} \alpha_i g(x_{i, 1}, x_{i, 2}) \]
where $x_{i, 1}, x_{i, 2}$ are the pair of features chosen to be part of the $i$th interaction and $\alpha_i$ are random coefficients drawn from a uniform distribution and normalized to sum to 1:
\[\sum_{i = 1}^{20} \alpha_i = 1\]

To ablate an interaction, we multiply the interaction in the label with gaussian noise, which ensures ablating the largest interactions adds the most amount of noise the the label. For example, let $f(x_1, x_2, x_3) = 0.5 * g(x_1, x_2) + 0.3 g(x_1, x_3)$. To ablate the interaction between $x_1$ and $x_2$, we compute the ablated label $\hat{f}(x_1, x_2, x_3) = \epsilon * 0.5 * g(x_1, x_2) + 0.3 g(x_1, x_3)$ where $\epsilon \sim \mathcal{N}(0, 1)$. The interaction method that identifies the largest interactions in the data will add the largest amount of random noise into the label, thus increasing error the fastest. \\

On the simulated data mentioned in the main text, we train a 3 layer neural network with 64 hidden units each and ReLU activation. It achieves near-perfect performance on the simulated regression task: explaining over 99\% of variance in the label. For each progressive ablation of an interaction, we retrain the network 5 times and re-evaluate performance. We then plot the mean and standard deviation of performance on a held-out set, as recommended by the original paper \citep{hooker2019benchmark}. In the main-text, we showed the results for $g_{\textrm{tanhsum}}(x_i, x_j) = \tanh(x_i + x_j)$. We show results for the four additional interaction types here:
\begin{itemize}
    \item $g_{\textrm{cossum}}(x_i, x_j) = \cos(x_i + x_j)$
    \item $g_{\textrm{multiply}}(x_i, x_j) = x_i * x_j$
    \item $g_{\textrm{maximum}}(x_i, x_j) = \max(x_i, x_j)$
    \item $g_{\textrm{minimum}}(x_i, x_j) = \min(x_i, x_j)$
\end{itemize}

The results are shown in Figure \ref{fig:additional_roar}. They show that our method consistently outperforms all methods except the monte-carlo estimation of the Shapley Interaction Index, which it performs just as well as our method. We note here that, as discussed in the main text, our method is much more computationally tractable than the Shapley Interaction Index, even using monte-carlo estimation. We also note here that Neural Interaction Detection \cite{tsang2017detecting} is not a local interaction method; rather, it detects interactions globally. To compare against it, we simply ablate the top-ranked interaction globally in all samples. 

\begin{figure}
    \centering
    \includegraphics[width=0.45\columnwidth]{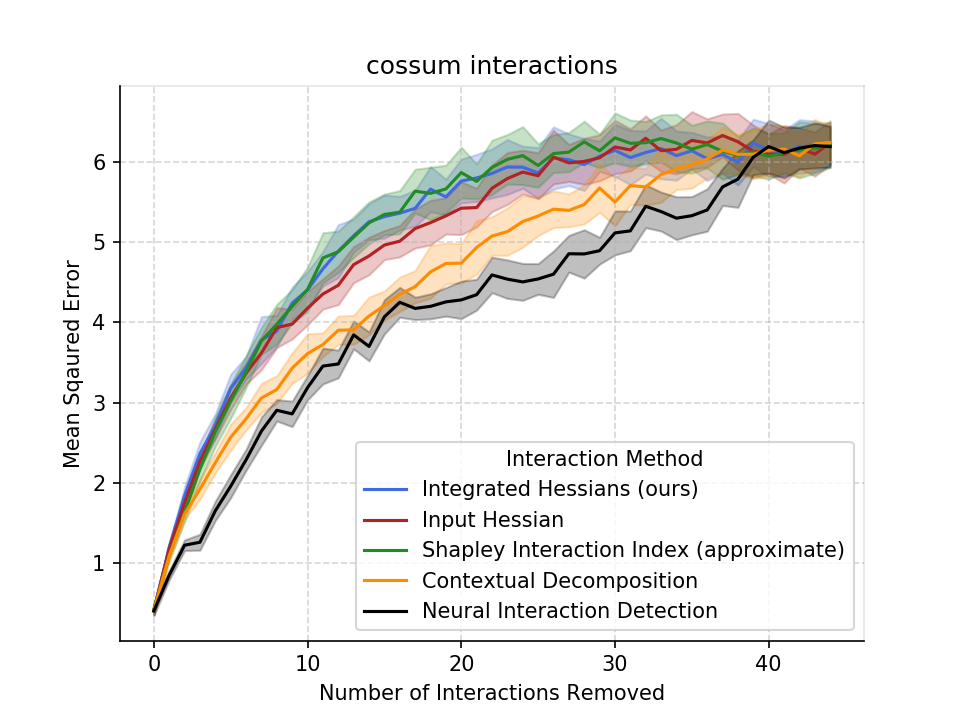}
    \includegraphics[width=0.45\columnwidth]{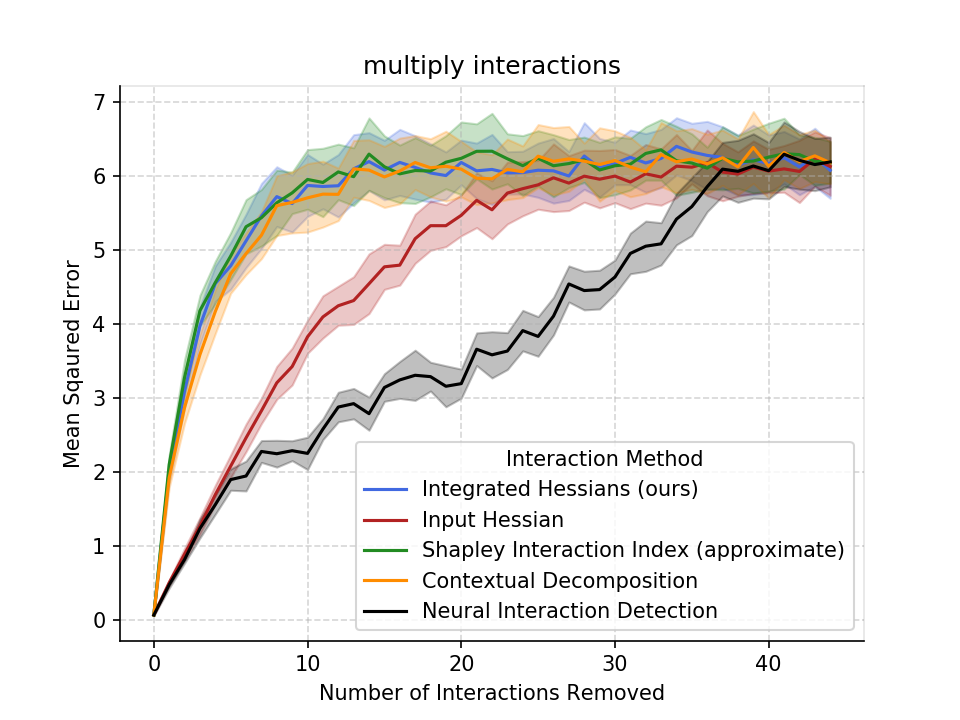}
    \includegraphics[width=0.45\columnwidth]{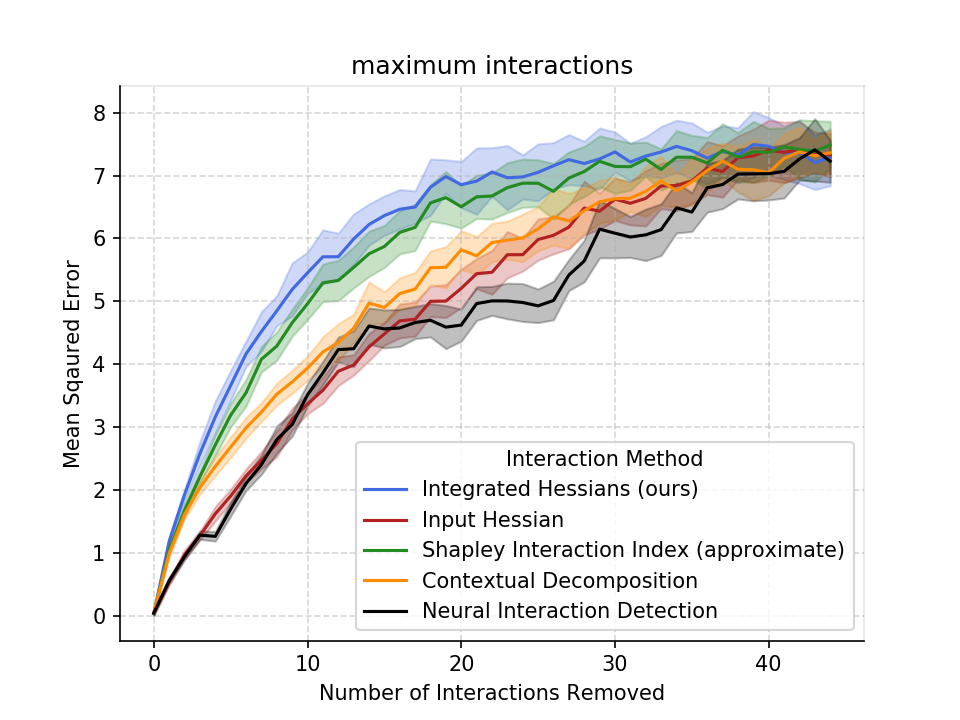}
    \includegraphics[width=0.45\columnwidth]{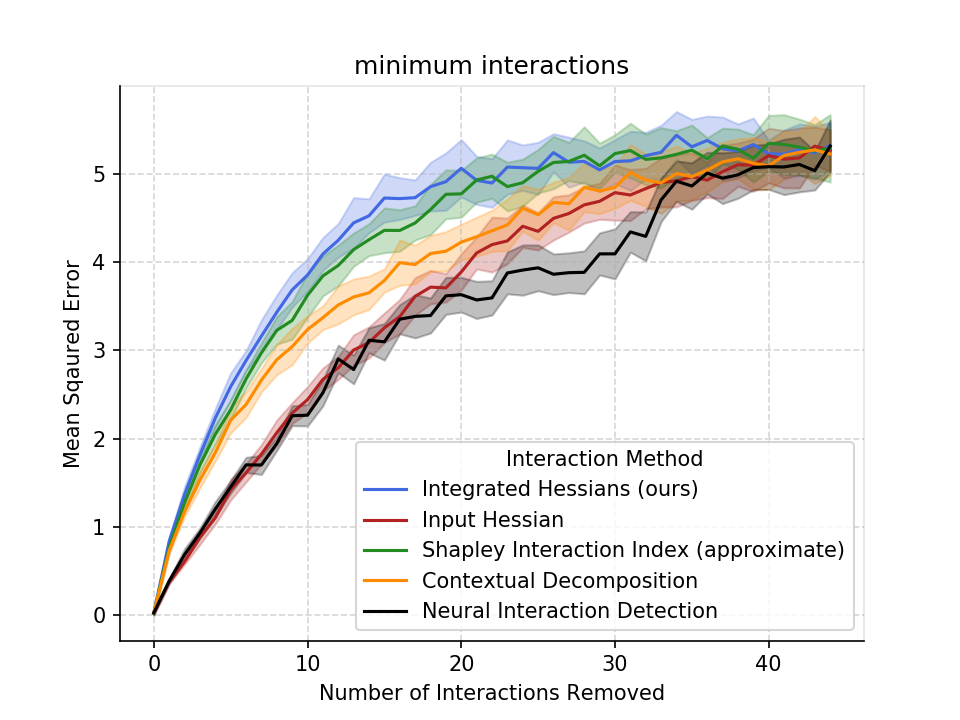}
    \caption{Additional comparisons against existing methods on known simulated interactions. The charts plot the mean and standard deviation across 5 re-trainings of the model as a function of ablated interactions. Our method (in blue) outperforms all existing methods on all interaction types. }
    \label{fig:additional_roar}
\end{figure}

\subsubsection{Rank Correlation and Sanity Checks}

We also quantitatively compared our method to other approaches by directly measuring the rank correlation between the true interactions and the interactions detected by each approach. We were able to do this because our data were synthetic. We created two synthetic datasets, both sharing the same set of features. The synthetic features consisted of five 0-mean unit-variance independent Gaussian random variables. For the first dataset (Multiplicative Interactions), the synthetic label was the sum of each of the ten possible pairwise multiplicative interactions between features with coefficients ranging in magnitude from 10 to 1:
\begin{equation}
    y = 10 x_1 x_2 - 9 x_1 x_3 + 8 x_1 x_4 + \cdots - 1 x_4 x_5.
\end{equation}

For the second dataset (Min/Max Interactions), the synthetic label was the sum of each of ten pairwise interactions, where each was a scalar coefficient times a pairwise minimum or maximum function:
\begin{equation}
    y = 10 \times \textrm{max}(x_1,x_2) - 9 \times \textrm{min}(x_1,x_3) + 8 \times \textrm{min}(x_1, x_4) - \cdots + \textrm{max}(x_4,x_5).
\end{equation}

For each dataset, we trained a neural network with two hidden layers and Tanh non-linearities until convergence (validation set $> R^2:0.99$). For all methods where local attributions could be attained (Integrated Hessians, Inputs * Hessian, CD, SII), we compared the true interaction contribution for each sample to the interaction detected by each attribution method. To translate local interactions into global interactions, we simply took the global interaction to be the average magnitude of interactions across all samples \citep{lundberg2020local,cui2019recovering}. Due to the small number of features (5) it was feasible to exhaustively compute the Shapley Interaction Index. We found that for both the multiplicative and non-multiplicative interactions, both globally and locally, Integrated Hessians was consistently the top performing method along with the Shapley Interaction Index (see \autoref{tab:rank_corr}).\\

\begin{table*}
    \centering
    \begin{tabular}{ |c|c|c|c|c| }
     \hline
     \textbf{Interaction Method} & \makecell{Multiplicative \\(Global)} & \makecell{Multiplicative \\(Local)} & \makecell{Non-Multiplicative\\ (Global)} & \makecell{Non-Multiplicative \\(Local)} \\
     \hline
     Integrated Hessians (ours) & \textbf{1.000} & \textbf{0.991}  &  \textbf{1.000} & \textbf{0.313} \\
     \hline
     SII \citep{grabisch1999axiomatic} & \textbf{1.000}  & \textbf{0.992}  &  \textbf{1.000} & 0.271 \\
     \hline
    %  Input * Hessian & 0.988  & 0.944  &  0.988 & -0.147\\
    %  \hline
     Hessian \citep{cui2019recovering} & \textbf{1.000}  & 0.008  &  \textbf{1.000} & -0.086 \\
     \hline
     NID \citep{tsang2017detecting} & 0.855  & n/a  &  0.878 & n/a \\
     \hline
     CD \citep{singh2018hierarchical} & 0.794  & 0.002  &  -0.285 & 0.000\\
     \hline
    \end{tabular}
    \caption{Rank correlation between each interaction detection method and true model interactions, both locally and globally for two different interaction types (multiplicative and non-multiplicative).}
    \label{tab:rank_corr}
\end{table*}

Finally, to ensure that our interaction attributions were sensitive to network and data randomization, we tested our approach using the ``Sanity Checks'' proposed in \cite{adebayo2018sanity}. Using the same data and network architecture described in the previous experiment, we first fit a network and found the interactions using Integrated Hessians. We then compared the rank correlation of these interactions with the interactions attained from explaining (a) a network with randomly initialized weights and (b) a network trained to convergence on the training set on data where the labels are shuffled at random (but the features are the same). We found that in \emph{both} settings, the Spearman correlation between the true and randomized interactions was 0.

\section{Effects of Smoothing ReLU Networks}

\subsection{Proof of Theorem 1}

\begin{theorem} For a one-layer neural network with softplus\textsubscript{$\beta$} non-linearity, $f_{\beta}(x) =$ softplus\textsubscript{$\beta$}$(w^Tx)$, and $d$ input features, we can bound the number of interpolation points $k$ needed to approximate the Integrated Hessians to a given error tolerance $\epsilon$ by $k \leq \mathcal{O}(\frac{d \beta^2}{\epsilon})$.
\end{theorem}

\begin{proof}
As pointed out in \citet{sundararajan2017axiomatic} and \citet{sturmfels2020visualizing}, completeness can be used to assess the convergence of the approximation. We first show that decreasing $\beta$ improves convergence for Integrated Gradients. In order to accurately calculate the Integrated Gradients value $\Phi$ for a feature $i$, we want to be able to bound the error between the approximate computation and exact value. The exact value is given as:

\begin{equation}
    \Phi_i(\theta,x,x') = (x_i - x_i') \times \int_{\alpha=0}^{1} \frac{ \partial f ( x'(1-\alpha) + x \alpha ) }{ \partial x_i} d\alpha.
\end{equation}

To simplify notation, we can define the partial derivative that we want to integrate over in the $i$th coordinate as $g_i(x) = \frac{\partial F(x)}{\partial x_i}$:

\begin{equation}
    \Phi_i(\theta,x,x') = (x_i - x_i') \times \int_{\alpha=0}^{1} g_i(x'\alpha + x(1-\alpha)) d\alpha.
\end{equation}

Since the single layer neural network with softplus activation is monotonic along the path, the error in the approximate integral can be lower bounded by the left Riemann sum $L_k$:

\begin{equation}
    L_k = \frac{\| x - x' \|}{k} \sum_{i=0}^{k-1} g_i (x' + \frac{i}{k}(x - x') ) \leq \int_{\alpha=0}^{1} g_i(x'\alpha + x(1-\alpha)) d\alpha.
\end{equation}

and can likewise be upper-bounded by the right Riemann sum $R_k$:

\begin{equation}
    \int_{\alpha=0}^{1} g_i(x'\alpha + x(1-\alpha)) d\alpha \leq  R_k = \frac{\| x - x' \|}{k} \sum_{i=1}^{k} g_i (x' + \frac{i}{k}(x - x') ). 
\end{equation}

We can then bound the magnitude of the error between the Riemann sum and the true integral by the difference between the right and left sums:

\begin{equation}
\label{eq:startOfDifference}
    \epsilon \leq | R_k - L_k | = \frac{\| x - x' \|_2}{k} | g_i (x) - g_i (x') |.
\end{equation}

By the mean value theorem, we know that for some $\eta \in [0,1]$ and $z = x' + \eta (x - x') $, $g_i (x) - g_i (x') = \nabla_x g_i(z)^\top (x - x')$. Therefore:
\begin{equation}
    \epsilon \leq \frac{\| x - x' \|}{k} \nabla_x g_i(z)^\top (x - x').
\end{equation}
Rewriting in terms of the original function, we have:
\begin{equation}
    \epsilon \leq \frac{\| x - x' \|}{k} \sum_{j=0}^{d} \bigg( \frac{\partial^2 f(z)}{\partial x_i x_j} (x_j - x_j') \bigg).
\end{equation}
We can then consider the gradient vector of $g_i(x)$:
\begin{equation}
\nabla_x g_i(x) = \bigg[\frac{\beta w_0 e^{\beta w^T x}}{(e^{\beta w^T x} + 1)^2} , \frac{\beta w_1 e^{\beta w^T x}}{(e^{\beta w^T x} + 1)^2} , \cdots \bigg]
\end{equation}
where each coordinate is maximized at the zeros input vector, and takes a maximum value of $\beta w_i / 4$. We can therefore bound the error in convergence as:
\begin{equation}
\label{eq:main}
    \epsilon \leq \frac{\| x - x' \|}{k} \sum_{j=0}^{d} \bigg( \frac{ \beta \|w\|_\infty}{4} (x_j - x_j') \bigg).
\end{equation}
Ignoring the dependency on path length and the magnitude of the weights of the neural network, we see that:
\begin{equation}
    k \leq \mathcal{O}\bigg(\frac{d\beta}{\epsilon}\bigg).
\end{equation}
This demonstrates that the number of interpolation points $k$ necessary to achieve a set error rate $\epsilon$ decreases as the activation function is smoothed (the value of $\beta$ is decreased). While this proof only bounds the error in the approximation of the integral for a single feature, we get the error in completeness by multiplying by an additional factor of $d$ features.\\

We can extend the same proof to the Integrated Hessians values. We first consider the error for estimating off diagonal terms $\Gamma_{i,j}, i \neq j$. The true value we are trying to approximate is given as:
\begin{equation}
    \Gamma_{ij} = (x_i - x_i')(x_j - x_j') \times \int_{\alpha\beta} \frac{ \partial^2 f (x' + \alpha \beta (x - x') ) }{\partial x_i \partial x_j} d \alpha d \beta
\end{equation}

For the sake of notation, we can say $h_{ij}(x) = \frac{ \partial^2 F (x) }{\partial x_i \partial x_j}$. Assuming that we are integrating from the all-zeros baseline as suggested in \citet{sundararajan2017axiomatic}, since $h_{ij} (x) = $ is monotonic on either interval from the 0 baseline, we can again bound the error in the double integral by the magnitude of the difference in the left and right Riemann sums:

\begin{equation}
    \epsilon \leq \bigg| \frac{ \| x - x' \|_2^2 }{k} \sum_{j=1}^{k} \sum_{i=1}^{k} h_{ij} (x' + \frac{ij}{k} (x - x') ) - \frac{ \| x - x' \|_2^2 }{k^2} \sum_{j=0}^{k-1} \sum_{i=0}^{k-1} h_{ij} (x' + \frac{ij}{k} (x - x') ) \bigg|
\end{equation}

\begin{equation}
    \epsilon \leq \frac{ \| x - x' \|_2^2 }{k} \bigg| \big( h_{ij} (x) + 2\sum_{i=1}^{k-1} h_{ij} (x' + \frac{i}{\sqrt{k}}(x - x') ) \big) - \big( h_{ij} (x') + 2\sum_{i=1}^{k-1} h_{ij} (x') \big) \bigg|
\end{equation}
We can then again use monotonicity over the interval to say that $h_{ij} (x' + \frac{i}{\sqrt{k}}(x - x') ) < h_{ij}(x)$, which gives us:
\begin{equation}
    \epsilon \leq \frac{(2k-1) \| x - x' \|_2^2 }{k} \bigg| h_{ij} (x)  - h_{ij} (x') \bigg|
\end{equation}
By the mean value theorem, we know that for some $\beta \in [0,1]$, $h_{ij}(x) - h_{ij}(x') = \nabla_x h_{ij}(\beta)^\top(x-x')$. Substituting gives us:
\begin{equation}
    \epsilon \leq \frac{(2k-1) \| x - x' \|_2^2 }{k} \nabla_x h_{ij}(\beta)^\top(x-x')
\end{equation}

We can then consider the elements of the gradient vector:
\begin{equation}
    \nabla_x h_{ij}(x) = \bigg[-\frac{\beta^2 w_i w_j w_1 e^{\beta w^T x} \big(e^{\beta w^T x} - 1 \big)}{(e^{\beta w^T x} + 1)^3}, -\frac{\beta^2 w_i w_j w_2 e^{\beta w^T x} \big(e^{\beta w^T x} - 1 \big)}{(e^{\beta w^T x} + 1)^3}, \cdots \bigg].
\end{equation}

For the univariate version of each coordinate, we can maximize the function by taking the derivative with respect to $x$ and setting it equal to 0.
\begin{equation}
    \frac{d}{dx} \bigg( -\frac{\beta^2 e^{\beta x} \big(e^{\beta x} - 1 \big)}{(e^{\beta x} + 1)^3} \bigg) = \frac{\beta^3 e^{\beta x}(-4 e^{\beta x} + e^{2\beta x} + 1)}{(e^{\beta x} + 1)^4} = 0
\end{equation}
% \begin{equation}
%     \frac{\beta^3 e^{\beta x}(-4 e^{\beta x} + e^{2\beta x} + 1)}{(e^{\beta x} + 1)^4} = 0
% \end{equation}
We can see that this equation holds only when $(-4 e^{\beta x} + e^{2\beta x} + 1)$ = 0, and can solve it by finding the roots of this quadratic equation, which occur when $x = \frac{1}{\beta}\log (2 \pm \sqrt{3})$. When we plug that back in, we find the absolute value of the function in that coordinate takes a maximum value of $\frac{\beta^2}{6 \sqrt{3}}$. Therefore, for a given set of fixed weights of the network, we can see that the coordinate-wise maximum magnitude of $\nabla_x h_{ij}, \propto \beta^2$, and that the number of interpolation points necessary to reach a desired level of error in approximating the double integral decreases as $\beta$ is decreased. Again ignoring the fixed weights and path length, the number of interpolation points necessary is bounded by:

\begin{equation}
    k \leq \mathcal{O}\bigg( \frac{d \beta^2}{k} \bigg).
\end{equation}

For the $i=j$ terms (main effect terms) the error will have another additive factor of $\beta$. This is because there is an added term to the main effect equal to:
\begin{equation}
    (x_i - x_i') \int_{\beta = 0}^{1} \int_{\alpha = 0}^{1} \frac{\partial f(x' + \alpha \beta (x - x'))}{\partial x_i} d\alpha d\beta
\end{equation}

When we bound the error in this approximate integral by the difference between the double left sum and double right sum, we get that:
\begin{equation}
    \epsilon \leq \frac{(2k-1) \| x - x' \|_2^2 }{k} \big| g_{i} (x)  - g_{i} (x') \big|
\end{equation}
Following the exact same steps as in \autoref{eq:startOfDifference} through \autoref{eq:main}, we can then show the bound on the error of the on-diagonal terms will have an additional term that is $\propto \beta$. Due to the axiom of interaction completeness, the error bound of the entire convergence can be obtained by adding up all of the individual terms, incurring another factor of $d^2$ in the bound.
% To consider each of the elements in the gradient vector, indexed by $k$, we first consider the case where $k = i$ or $k = j$...
% \begin{equation}
%     \frac{\partial h_{ij}}{\partial x_k} = -\frac{\beta^2 w_i^2 w_j e^{\beta w^T x} \big(e^{\beta w^T x} - 1 \big)}{(e^{\beta w^T x} + 1)^3}
% \end{equation}
% Then we consider the case where $k \neq i$ and $k \neq j$...
% \begin{equation}
%     \frac{\partial h_{ij}}{\partial x_k} = -\frac{\beta^2 w_i w_j w_k e^{\beta w^T x} \big(e^{\beta w^T x} - 1 \big)}{(e^{\beta w^T x} + 1)^3}
% \end{equation}
% The maximum of this gradient vector is...
% To consider each of the elements in the gradient vector, indexed by $k$, we first consider the case when $i = j = k$...
% \begin{multline}
%     \frac{\partial h_{ij}}{\partial x_k} = \frac{\beta^3 w_i^3}{\beta} \bigg[ - \frac{e^{-\beta w^T x}}{e^{-\beta w^T x} + 1} \\ + 3 \frac{e^{-2\beta w^T x}}{(e^{-\beta w^T x} + 1)^2} - 2 \frac{e^{-3\beta w^T x}}{(e^{-\beta w^T x} + 1)^3} \bigg]
% \end{multline}

% \begin{equation}
%     \frac{ \| x - x' \|_2^2 }{mn} \bigg( h_{ij} (x) - h_{ij} (x') \bigg)
% \end{equation}

% \begin{equation}
%     \frac{ \| x - x' \|_2^2 }{mn} \sum_{j=0}^{m} \sum_{i=0}^n h_{ij} (x' + \frac{i}{n}\frac{j}{m} (x - x') )
% \end{equation}

\end{proof}

\subsection{SoftPlus Smoothing Empirically Improves Convergence}

In addition to theoretically analyzing the effects of smoothing the activation functions of a single-layer neural network on the convergence of the approximately calculation of Integrated Gradients and Integrated Hessians, we also wanted to empirically analyze the same phenomenon in deeper networks. We assessed this by creating two networks: one with 5 hidden layers of 50 nodes, and a second with 10 hidden layers of 50 nodes. These networks were then randomly initialized using the Xavier Uniform intialization scheme \citep{glorot2010understanding}. We created 10 samples to explain, each with 100 features drawn at random from the standard normal distribution. To evaluate the convergence of our approximate Integrated Hessians values, we plot the interaction completeness error (the difference between the sum of the Integrated Hessians value and the difference of the function output at a sample and the function output at the zeros baseline). We plot the completeness error as a fraction of the magnitude of the function output. As we decrease the value of $\beta$, we smooth the activations, and we can see that the number of interpolations required to converge decreases (see \autoref{fig:Completeness5} and \autoref{fig:Completeness10}). We note that the randomly initialized weights of each network are held constant and the only thing changed is the value of $\beta$ in the activation function.

\begin{figure}
    \centering
    \includegraphics[width=0.5\columnwidth]{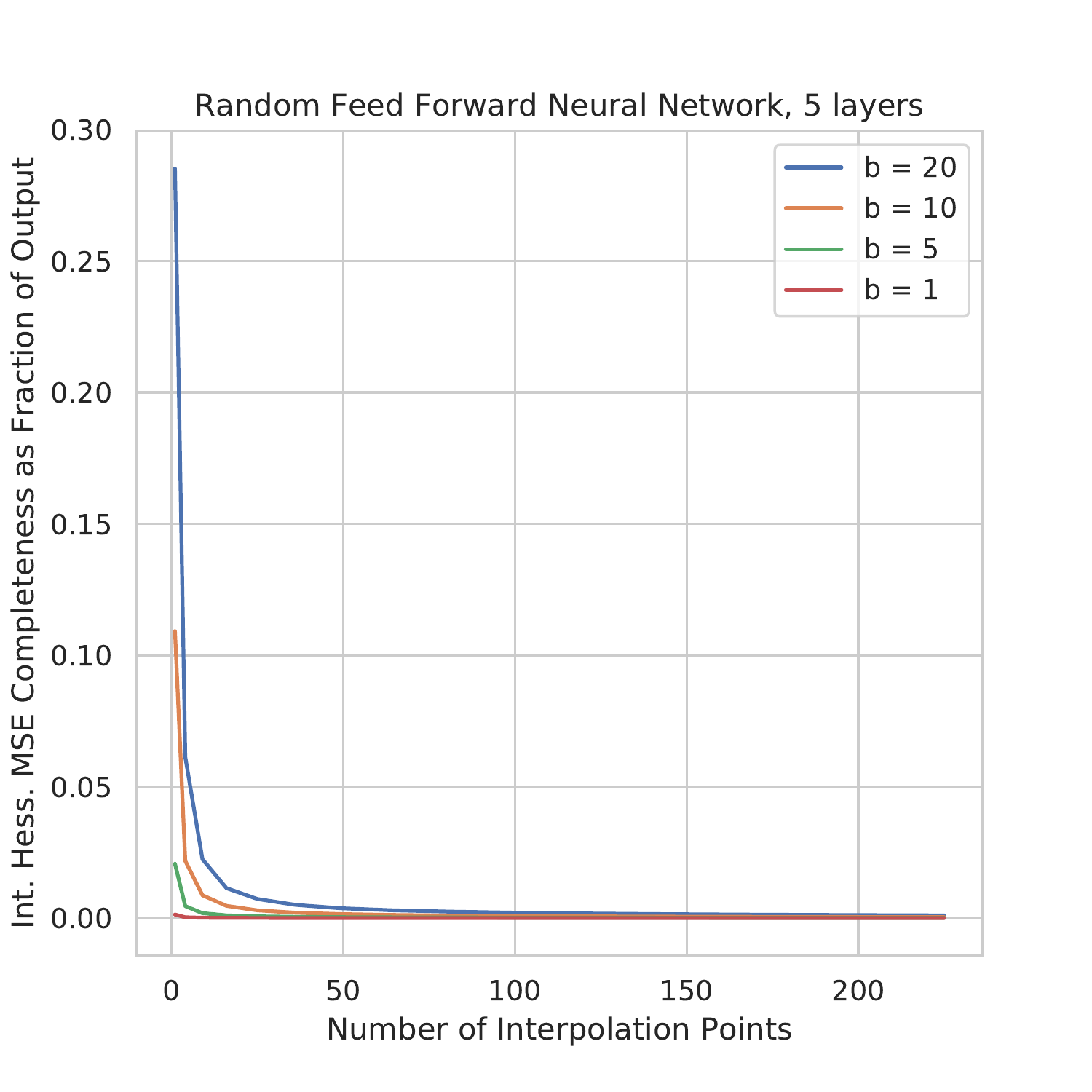}
    \caption{5-Layer Network Results. Interaction completeness error (difference between model output and sum of Integrated Hessians values) decreases more quickly with number of interpolation points as the $\beta$ parameter for the softplus activation function is decreased (as the function is smoothed). Results are averaged over 10 samples with 100 features for a neural network with 5 hidden layers of 50 nodes each.}
    \label{fig:Completeness5}
\end{figure}

\begin{figure}
    \centering
    \includegraphics[width=0.5\columnwidth]{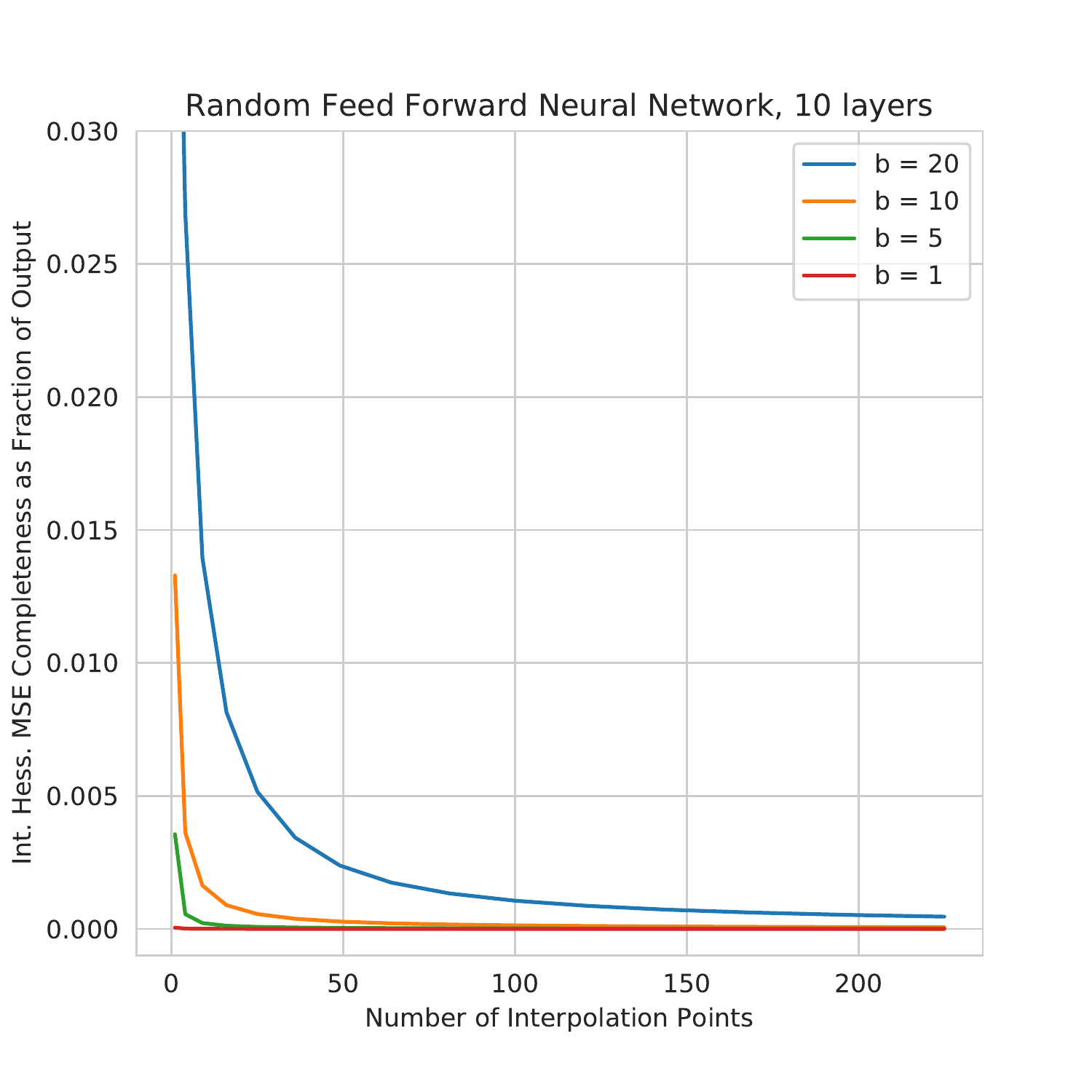}
    \caption{10-Layer Network Results. We can see that decreasing $\beta$ has an even more dramatic effect on convergence in this case when the network is deeper than in the 5-layer case. Again, this result shows that the interaction completeness error decreases more quickly with number of interpolation points as the activation function is smoothed. Results are averaged over 10 samples with 100 features for a neural network with 10 hidden layers of 50 nodes each.}
    \label{fig:Completeness10}
\end{figure}

\section{Details on the Sentiment Analysis Task}
\subsection{Fine-Tuning DistilBERT}

As mentioned in the main text, we download pre-trained weights for DistilBERT, a pre-trained language model introduced in \citet{sanh2019distilbert}, from the HuggingFace Transformers library \cite{Wolf2019HuggingFacesTS}. We fine-tune the model on the Stanford Sentiment Treebank dataset introduced by \citet{socher2013recursive}. We fine-tune for 3 epochs using a batch size of 32 and a learning rate of 0.00003. We use a a max sequence length of 128 tokens, and the Adam algorithm for optimization \cite{kingma2014adam}. We tokenize using the HuggingFace build-in tokenizer, which does so in an uncased fashion. We did not search for these hyper-parameters - rather, they were the defaults presented for fine-tuning in the HuggingFace repository. We find that they work adequately for our purposes, and didn't attempt to search through more hyper-parameters. 

\subsection{Training a CNN}
The convolutional neural network that we compare to in the main text is one we train from scratch on the same dataset. We randomly initialize 32-dimensional embeddings and use a max sequence length of 52. First, we apply dropout to the embeddings with dropout rate 0.5. The network itself is composed of 1D convolutions with 32 filters of size 3 and 32 filters of size 8. Each filter size is applied separately to the embedding layer, after which max pooling with a stride of 2 is applied and then the output of both convolutions is concatenated together and fed through a dropout layer with a dropout rate of 0.5 during training. A hidden layer of size 50 follows the dropout, finally followed by a linear layer generating a scalar prediction that the sigmoid function is applied to. \\

We train with a batch size of 128 for 2 epochs and use a learning rate of 0.001. We optimize using the Adam algorithm with the default hyper-parameters \cite{kingma2014adam}. Since this model was not pre-trained on a large language corpus and lacks the expressive power of a deep transformer, it is unable to capture patterns like negation that a fine-tuned DistilBERT does.

\subsection{Generating Attributions and Interactions}

In order to generate attributions and interactions, we use Integrated Gradients and Integrated Hessians with the zero-embedding - the embedding produced by the all zeros vector, which normally encodes the padding token - as a baseline. Because embedding layers are not differentiable, we generate attributions and interactions to the word embeddings and then sum over the embedding dimension to get word-level attributions and interactions, as done in \citet{sundararajan2017axiomatic}. When computing attributions and interactions, we use 256 background samples. Because DistilBERT uses the GeLU activation function \cite{ramachandran2017searching}, which has continuous first and second partial derivatives, there is no need to use the softplus replacement. When we plot interactions, we avoid plotting the main-effect terms in order to better visualize the interactions between words. 

\subsection{Additional Examples of Interactions}

Here we include additional examples of interactions learned on the sentiment analysis task. First we expand upon the idea of saturation in natural language, displayed in Figure \ref{fig:increasing_sat}. We display interactions learned by a fine-tuned DistilBERT on the following sentences: ``a bad movie'' (negative with 0.9981 confidence), ``a bad, terrible movie'' (negative with 0.9983 confidence), ``a bad, terrible, awful movie'' (negative with 0.9984 confidence) and ``a bad, terrible, awful, horrible movie'' (negative with 0.9984 confidence). The confidence of the network saturates: a network output only gets so negative before it begins to flatten. However, the number of negative adjectives in the sentence increases. This means a sensible network would spread the same amount of credit (because the attributions sum to the saturated output) across a larger number of negative words, which is exactly what DistilBERT does. However, this means that each word gets less negative attribution than it would if it was on its own. Thus, the negative words have positive interaction effects, which is exactly what we see from the figure.\\

In Figure \ref{fig:beautifully_observed}, we give another example of the full interaction matrix on a sentence from the validation set. In Figure \ref{fig:mainly_nonprofessional}, we give an example of how explaining the importance of a particular word can understand whether that word is important because of its main effect or because of its surrounding context. We show additional examples from the validation set in Figures \ref{fig:better_suited}, \ref{fig:fincher}, \ref{fig:good_script}, \ref{fig:little_to_love}, \ref{fig:quite_compelling}. We note that while some interactions make intuitive sense to humans (``better suited'' being negative or ``good script'' being positive), there are many other examples of interactions that are less intuitive. These interactions may indicate that the Stanford Sentiment Treebank dataset does not fully capture the expressive power of language (e.g. it doesn't have enough samples to fully represent all of the possible interactions in language), or it may indicate that the model has learned higher order effects that cannot be explained by pairwise interactions alone.

\begin{figure}
    \centering
    \begin{subfigure}
      \centering
      \includegraphics[width=.4\columnwidth]{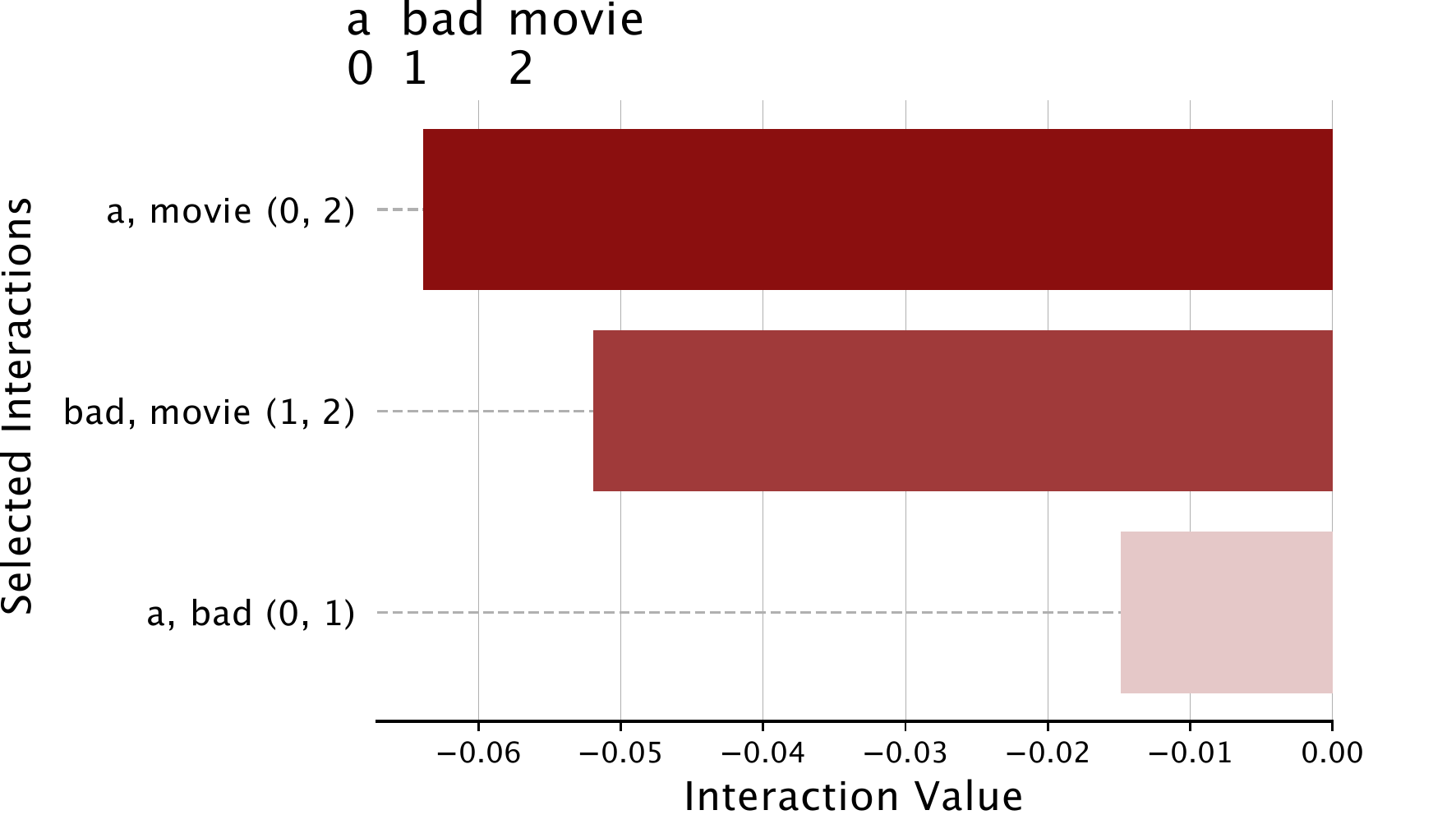}
    \end{subfigure}
    \begin{subfigure}
      \centering
      \includegraphics[width=.4\columnwidth]{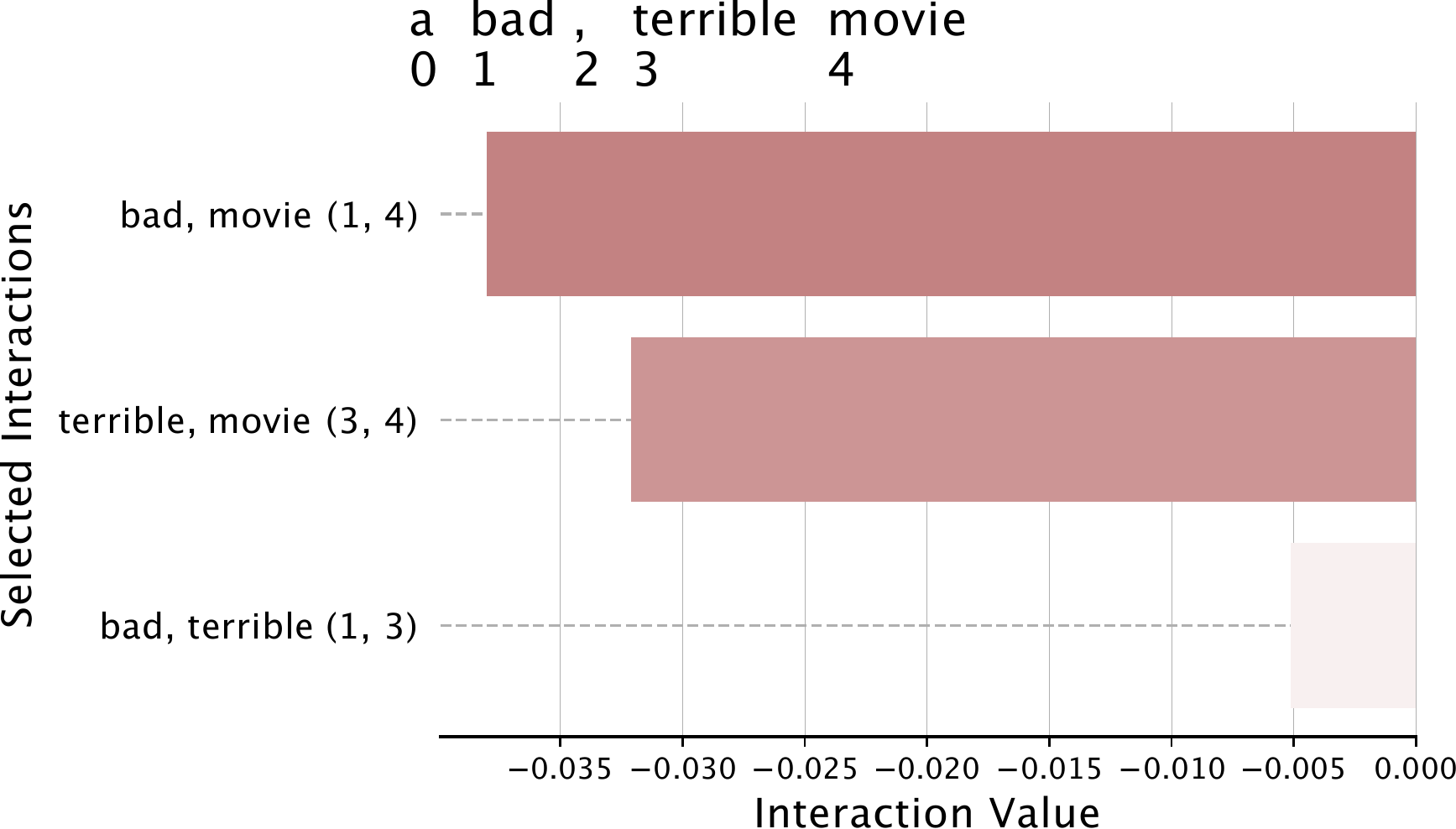}
    \end{subfigure}
    
    \begin{subfigure}
      \centering
      \includegraphics[width=.4\columnwidth]{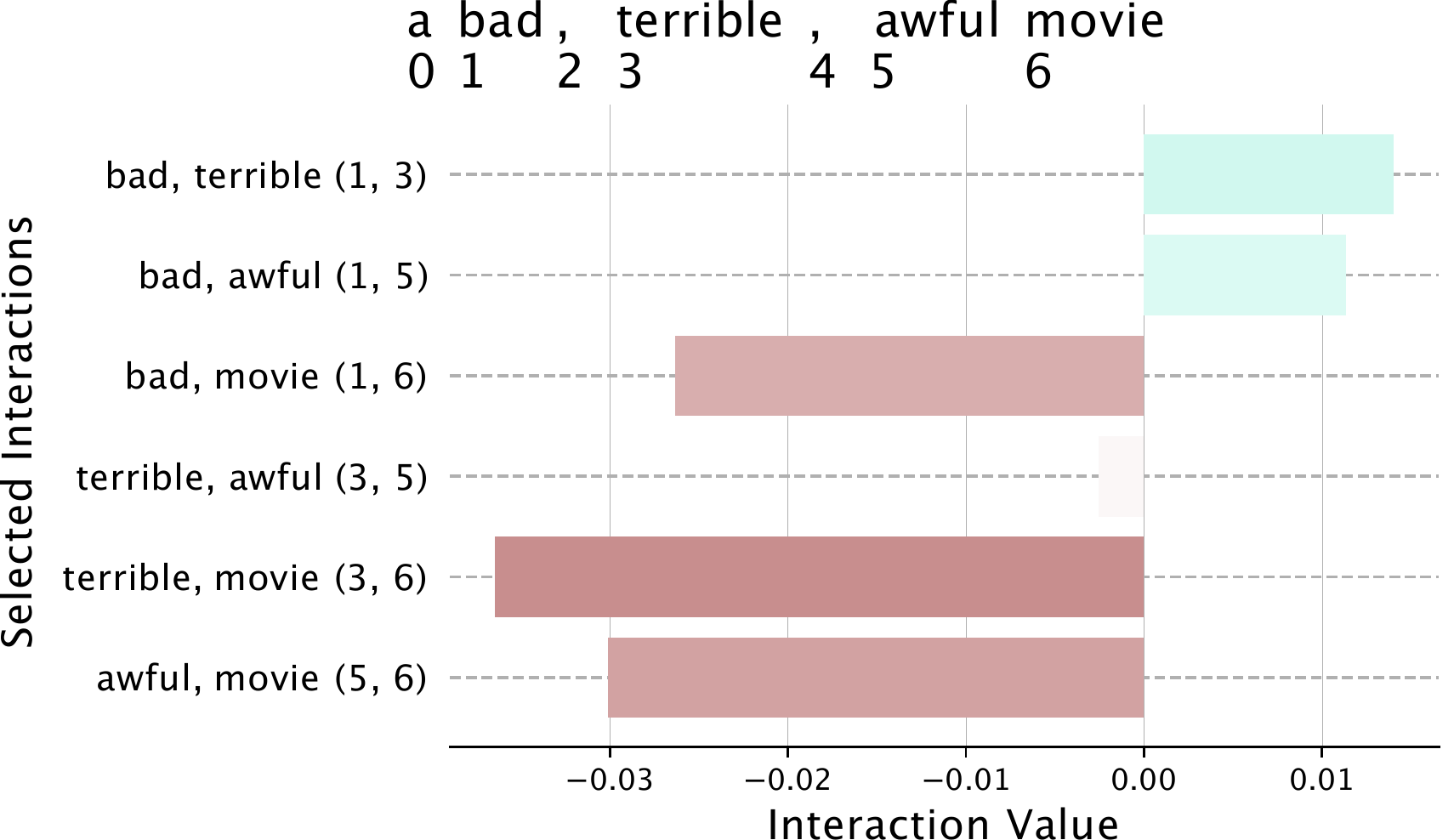}
    \end{subfigure}
    \begin{subfigure}
      \centering
      \includegraphics[width=.4\columnwidth]{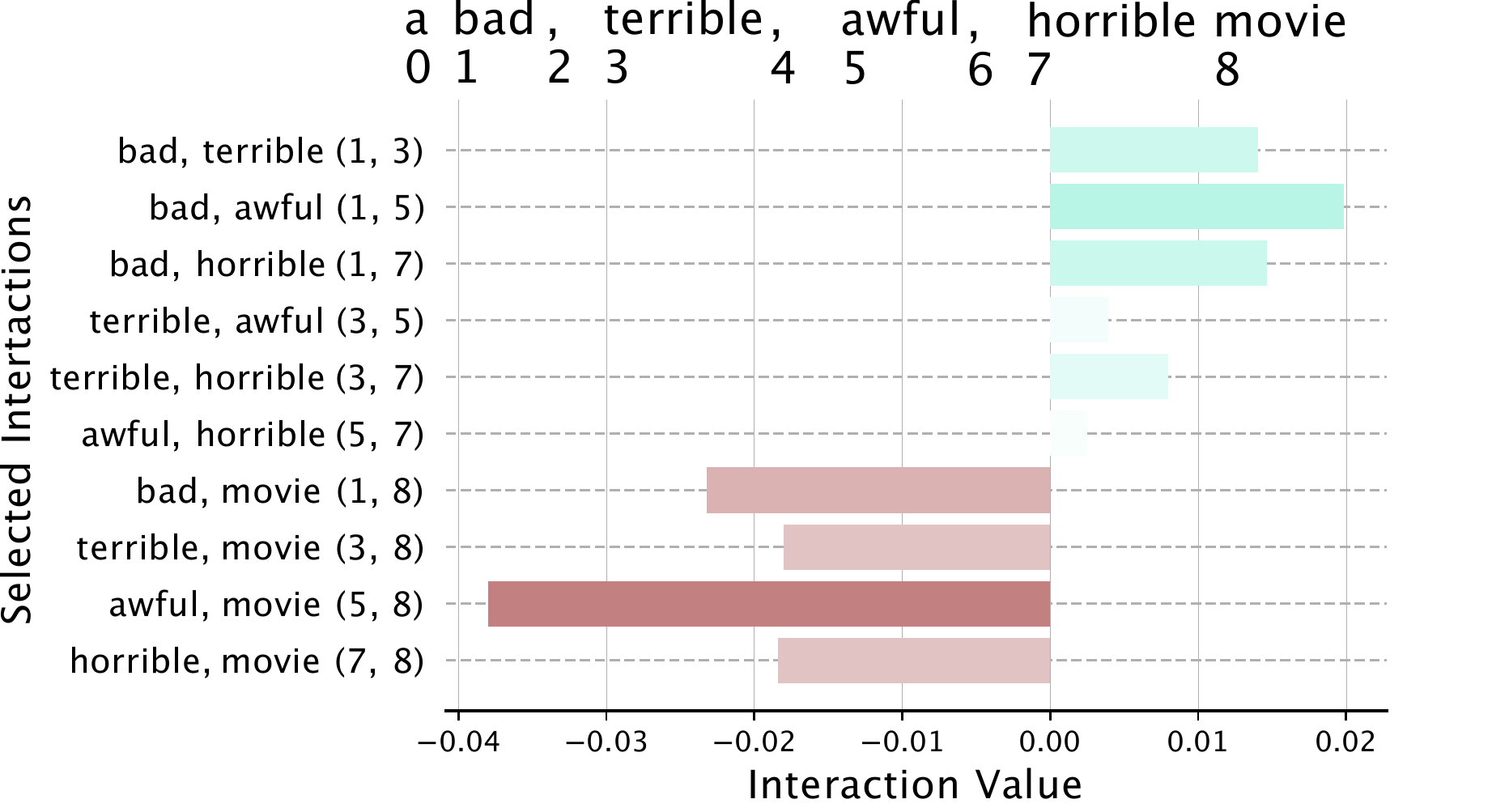}
    \end{subfigure}
    
    \caption{The effects of increasing saturation. As we add more negative adjectives to describe the word ``movie'', those negative adjectives
    interact more and more positively, even though those words interact negatively with the word they are describing. This is because each individual negative adjective has less impact on the overall negativity of the sentence the more negative adjectives there are.
    }
    
    \label{fig:increasing_sat}
\end{figure}

\begin{figure}
    \centering

    \includegraphics[width=.7\columnwidth]{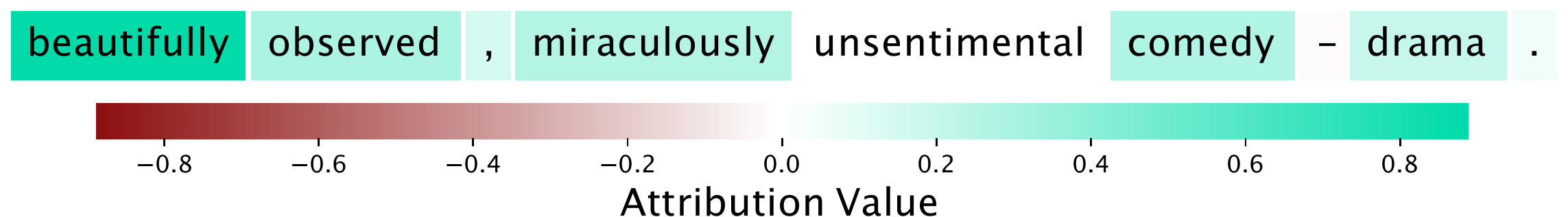}
    \includegraphics[width=.7\columnwidth]{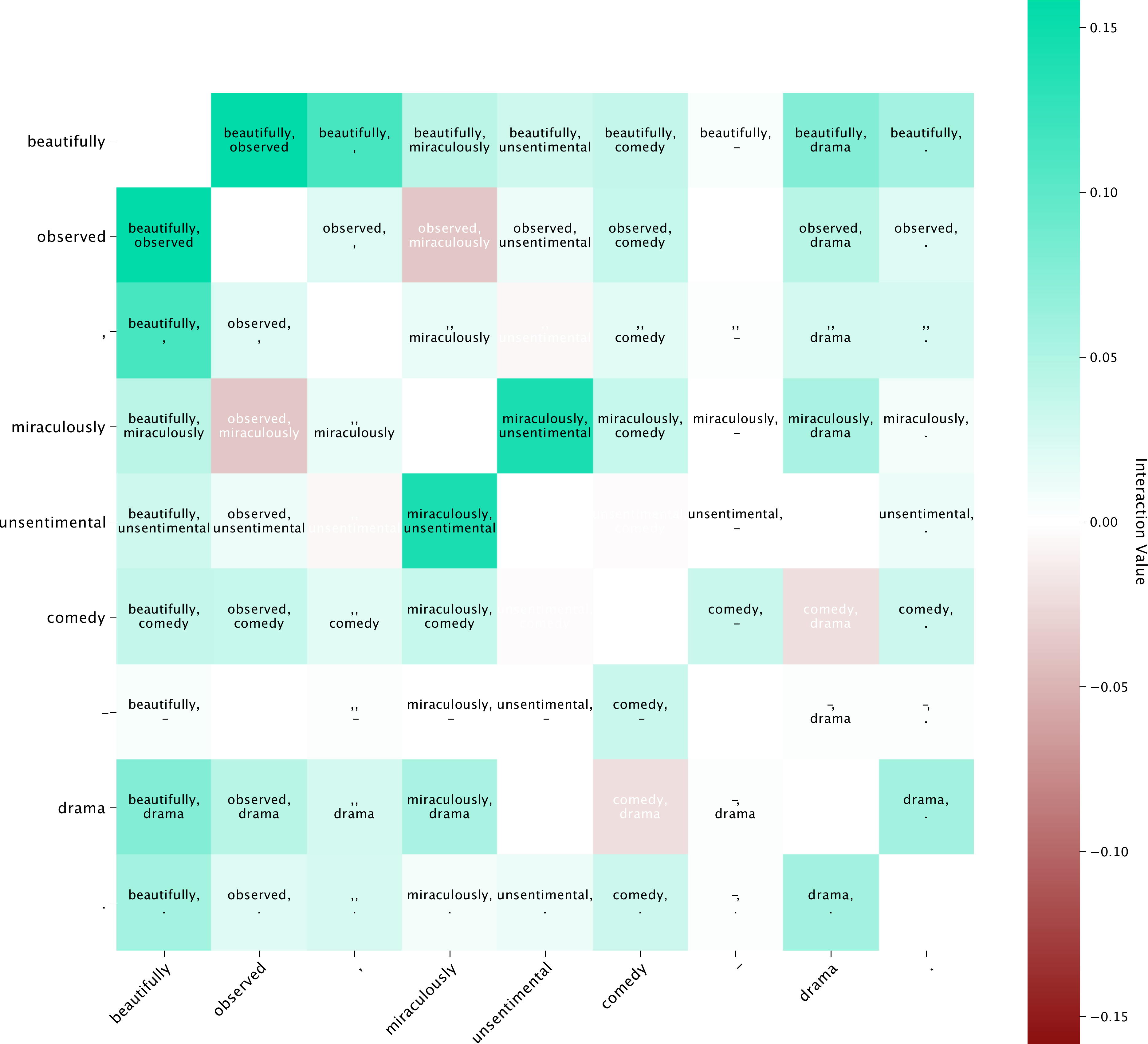}
    
    \caption{An example from the Stanford Sentiment Analysis Treebank validation set. Interactions highlight intuitive patterns in text, such as phrases like "beautifully observed" and "miraculously unsentimental" being strongly positive interactions. }
    
    \label{fig:beautifully_observed}
\end{figure}

\begin{figure}
    \centering

    \includegraphics[width=.9\columnwidth]{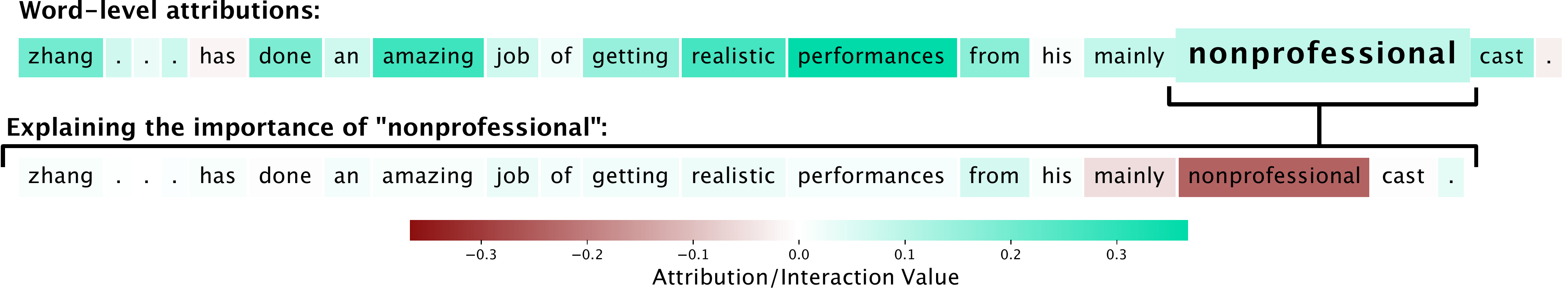}
    
    \caption{An example from the Stanford Sentiment Analysis Treebank validation set. This example shows that the word "nonprofessional" has a main effect that is negative, but the surrounding context outweights the main effect and makes the overall attribution positive. }
    
    \label{fig:mainly_nonprofessional}
\end{figure}

\begin{figure}
    \centering

    \includegraphics[width=.7\columnwidth]{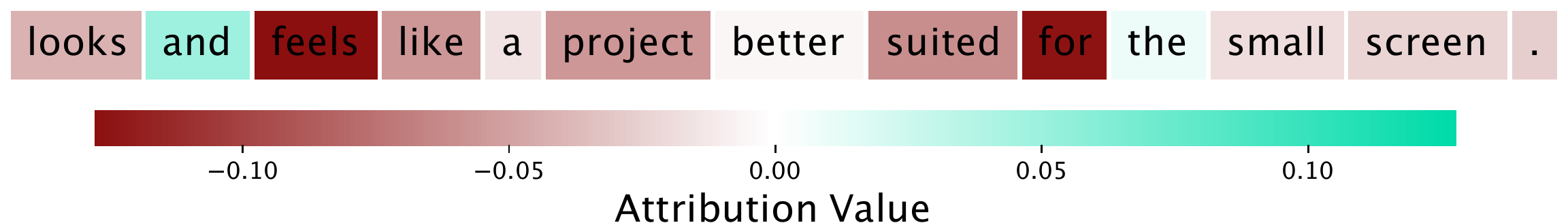}
    \includegraphics[width=.7\columnwidth]{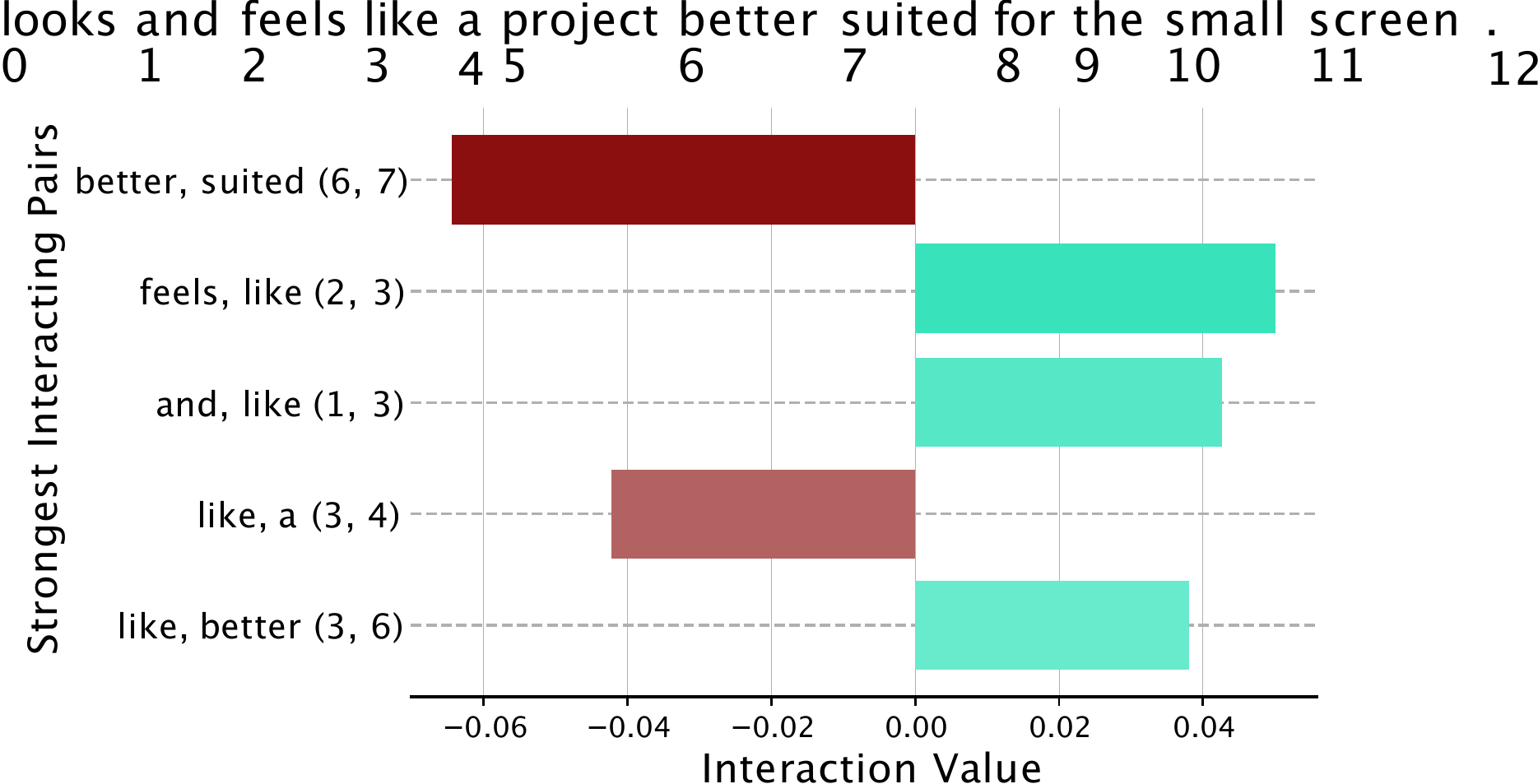}
    
    \caption{An example from the Stanford Sentiment Analysis Treebank validation set. This example highlights the phrase ``better suited'' being strongly negative, which is very intuitive. However, some of the other interactions are slightly less intuitive and may indicate lack of training data or higher-order interactions beyond word pairs. }
    
    \label{fig:better_suited}
\end{figure}

\begin{figure}
    \centering

    \includegraphics[width=.7\columnwidth]{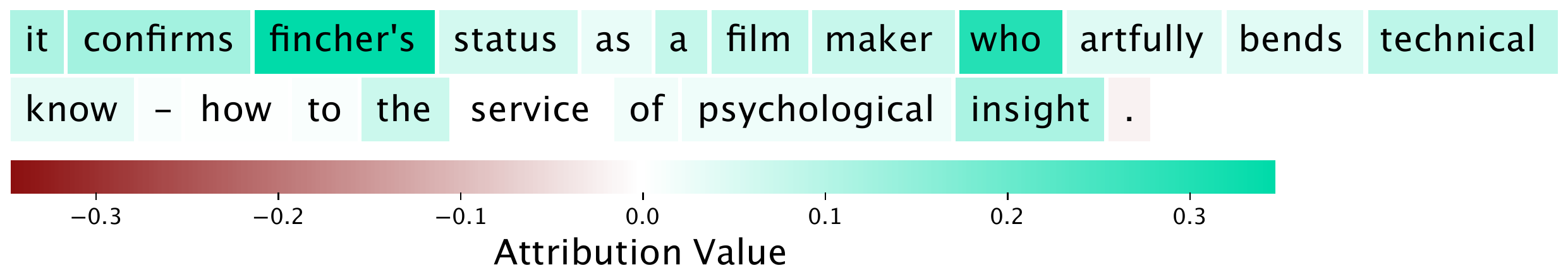}
    \includegraphics[width=.7\columnwidth]{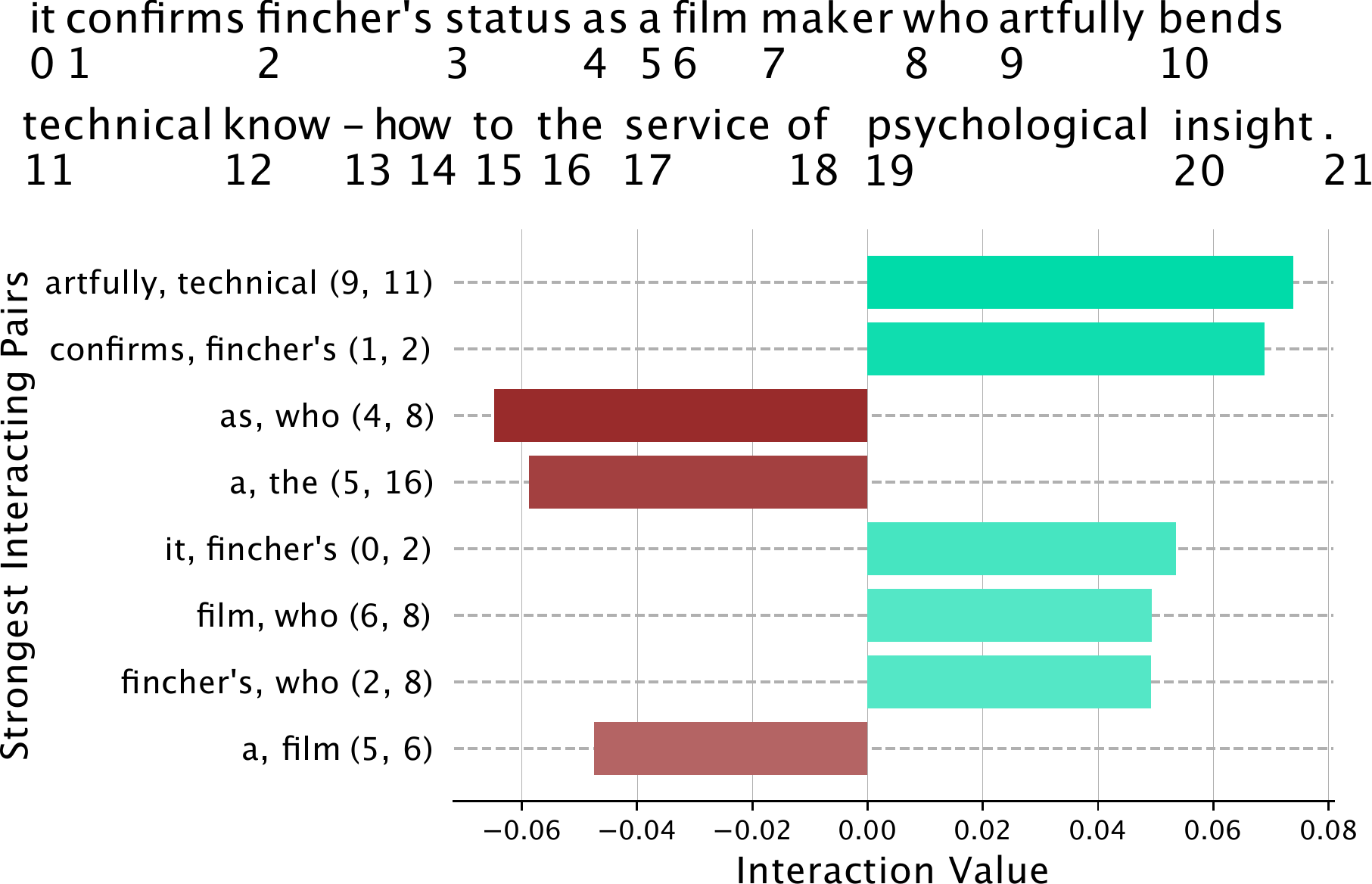}
    
    \caption{Another example from the Stanford Sentiment Analysis Treebank validation set. Notice how the strongest interact pairs may not necessarily be adjacent words, e.g. ``artfully'' and ``technical''. }
    
    \label{fig:fincher}
\end{figure}

\begin{figure}
    \centering

    \includegraphics[width=.7\columnwidth]{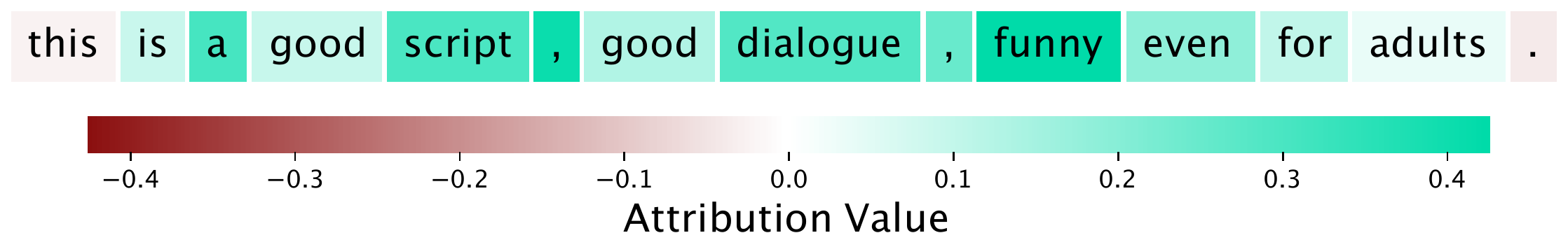}
    \includegraphics[width=.7\columnwidth]{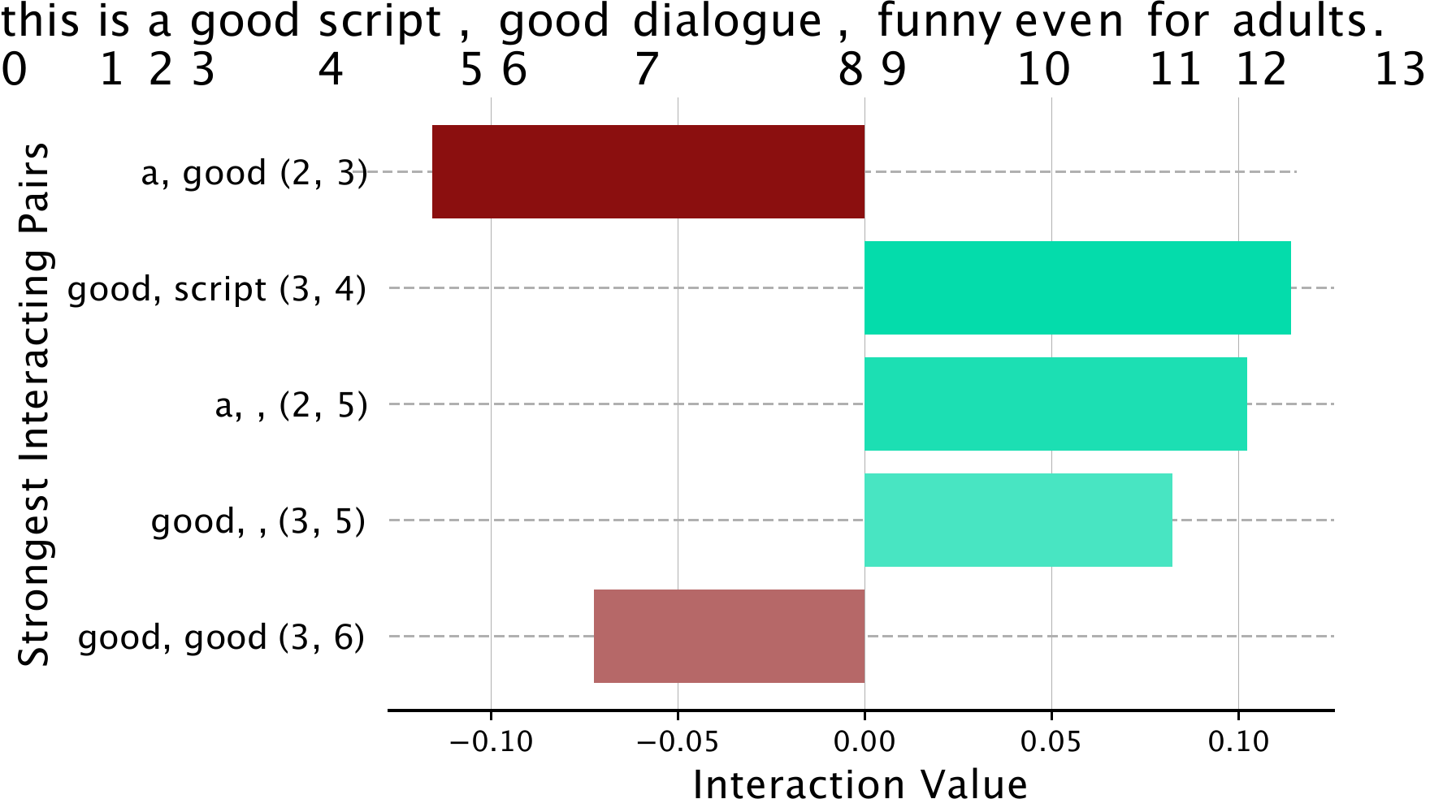}
    
    \caption{An example from the Stanford Sentiment Analysis Treebank validation set. Interestingly, the phrase ``a good'' has a negative interaction. This may indicate saturation effects, higher order effects, or that the model has simply learned an unintuitive pattern. }
    
    \label{fig:good_script}
\end{figure}

\begin{figure}
    \centering

    \includegraphics[width=.7\columnwidth]{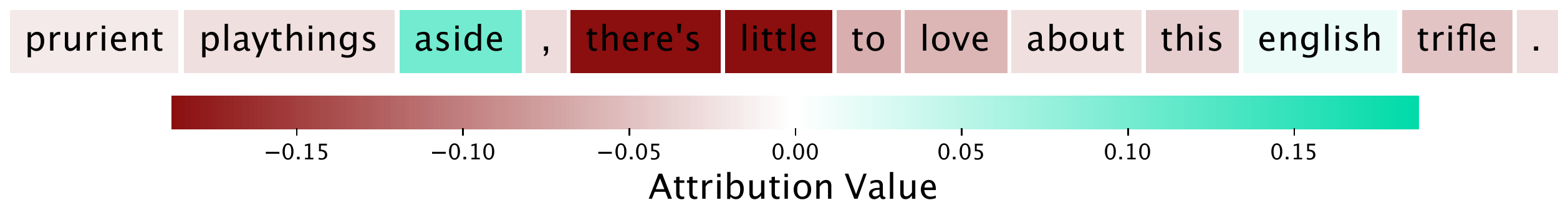}
    \includegraphics[width=.7\columnwidth]{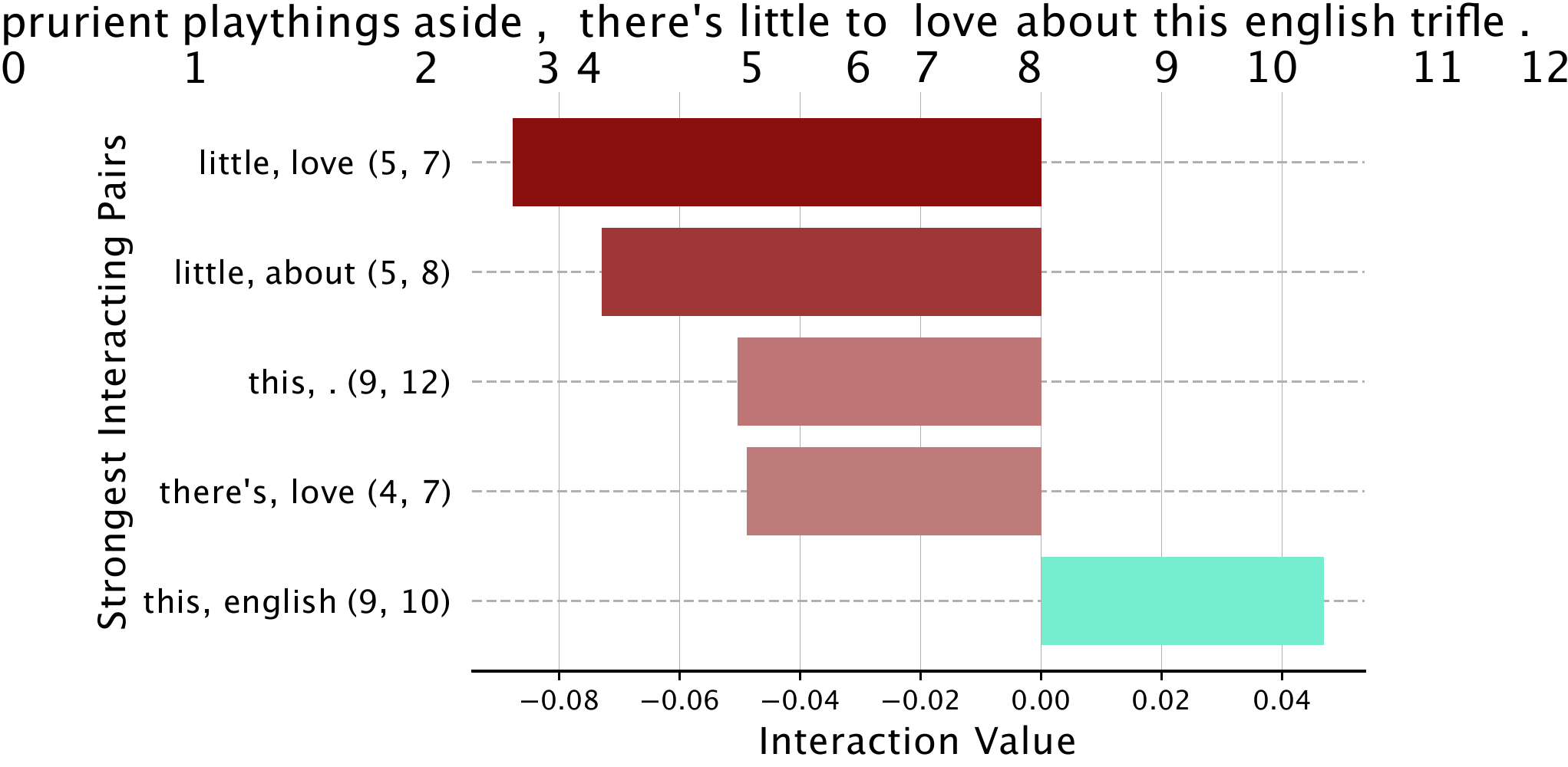}
    
    \caption{An example from the Stanford Sentiment Analysis Treebank validation set. This example shows some very intuitive negative interactions among the phrase ``there's little to love''. Interestingly, ``this english'' has a positive interaction: perhaps the dataset has a bias for english movies? }
    
    \label{fig:little_to_love}
\end{figure}

\begin{figure}
    \centering

    \includegraphics[width=.7\columnwidth]{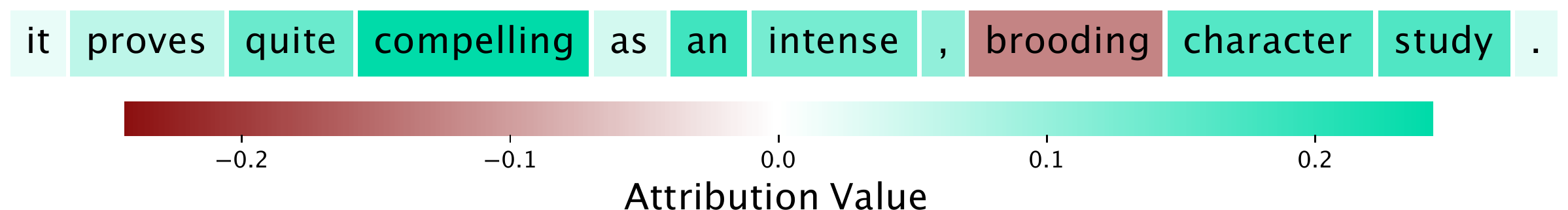}
    \includegraphics[width=.7\columnwidth]{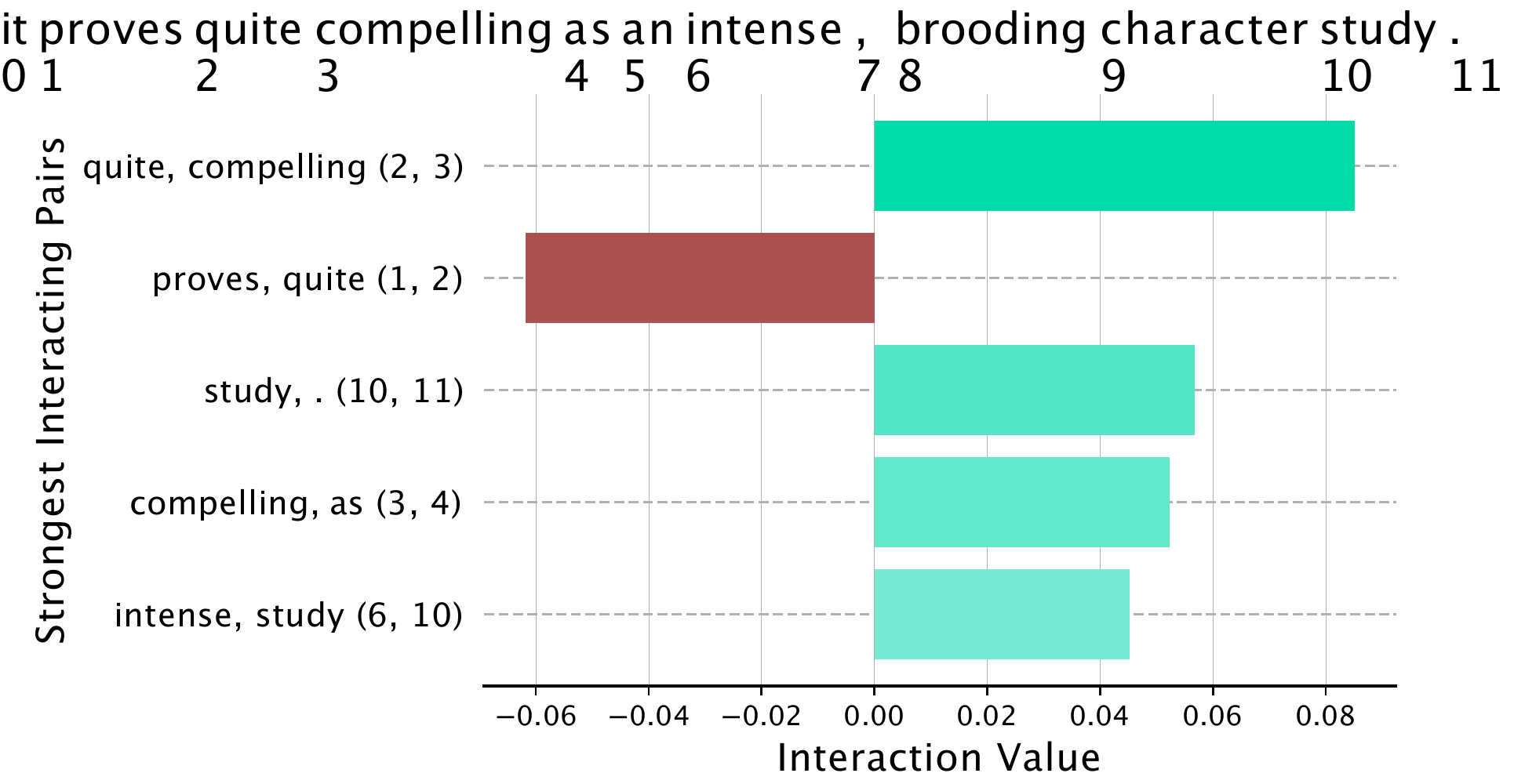}
    
    \caption{An example from the Stanford Sentiment Analysis Treebank validation set. This one also shows intuitive patterns, e.g. ``quite compelling'' being strongly positive. }
    
    \label{fig:quite_compelling}
\end{figure}

\section{Additional Experiments}
\subsection{Heart disease prediction}
\label{sec:heart_disease}

Here we aggregate interactions learned from many samples in a clinical dataset and use the interactions to reveal global patterns. We examine the Cleveland heart disease dataset \citep{detrano1989international, das2009effective}. After preprocessing, the dataset contains 298 patients with 13 associated features, including demographic information like age and gender and clinical measurements such as systolic blood pressure and serum cholesterol. The task is to predict whether or not a patient has coronary artery disease. The list of features, which we reproduce here, is from \cite{detrano1989international}, the original paper introducing the dataset:
\begin{enumerate}
    \item Age of patient (mean: 54.5 years $\pm$ standard deviation: 9.0)
    \item Gender (202 male, 96 female)
    \item Resting systolic blood pressure (131.6 mm Hg $\pm$ 17.7)
    \item Cholesterol (246.9 mg/dl $\pm$ 51.9)
    \item Whether or not a patient's fasting blood sugar was above 120 mg/dl (44 yes)
    \item Maximum heart rate achieved exercise (149.5 bpm $\pm$ 23.0)
    \item Whether or not a patient has exercise-induced angina (98 yes)
    \item Excercise-induced ST-segment depression (1.05 mm $\pm$ 1.16)
    \item Number of major vessels appearing to contain calcium as revealed by cinefluoroscopy (175 patients with 0, 65 with 1, 38 with 2, 20 with 3)
    \item Type of pain a patient experienced if any (49 experienced typical anginal pain, 84 experienced atypical anginal pain, 23 experienced non-anginal pain and 142 patients experienced no chest pain)
    \item Slope of peak exercise ST segment (21 patients had upsloping segments, 138 had flat segments, 139 had downsloping segments)
    \item Whether or not a patient had thallium defects as revealed by scintigraphy (2 patients with no information available, 18 with fixed defects, 115 with reversible defects and 163 with no defects)
    \item Classification of resting electrocardiogram (146 with normal resting ecg, 148 with an ST-T wave abnormality, and 4 with probable or definite left centricular hypertrophy)
\end{enumerate}

We split the data into 238 patients for training (of which 109 have coronary artery disease) and 60 for testing (of which 28 have coronary artery disease). We use a two layer neural network with 128 and 64 hidden units, respectively, with softplus activation after each layer. We optimize using gradient descent (processing the entire training set in a single batch) with an initial learning rate of 0.1 that decays exponentially with a rate 0.99 after each epoch. We use nesterov momentum with $\beta = 0.9$ \citep{sutskever2013importance}. After training for 200 epochs, the network achieves a held-out accuracy of 0.8667 with 0.8214 true positive rate and 0.9062 true negative rate. We note that the hyper-parameters chosen here were not carefully tuned on a validation set - they were simply those that seemed converge to a reasonable performance on the training set. Our focus is not state of the art prediction or comparing model performances, but rather interpreting patterns a reasonable model learns. \\

To generate attributions and interactions for this dataset, we use Expected Gradients and Expected Hessians with the training set forming the background distribution. We use 200 samples to compute both attributions and interactions, although we note this number is probably larger than necessary but was easy to compute due to the small size of the dataset. \\

Figure \ref{fig:heart_disease_summary} shows which features were most important towards predicting heart disease aggregated over the entire dataset, as well as the trend of importance values. Interestingly, the model learns some strangely unintuitive trends: if a patient doesn't experience chest pain, they are more likely to have heart disease than if they experience anginal chest pain! This could indicate problems with the way certain features were encoded, or perhaps dataset bias. Figure \ref{fig:max_heart_rate} demonstrates an interaction learned by the network between maximum heart rate achieved and gender, and Figure \ref{fig:st_depression} demonstrates an interaction between exercise-induced ST-segment depression and number of major vessels appearing to contain calcium. \\

In Figure \ref{fig:major_vessels}, we examine the interactions with a feature describing the number of major coronary arteries with calcium accumulation (0 to 3), as determined by cardiac cinefluoroscopy \citep{DETRANO19861041}. Previous research has shown that this technique is a reliable way to gauge calcium build-up in major blood vessels, and serves as a strong predictor of coronary artery disease \citep{DETRANO19861041, bartel1974significance, liu2015current}. Our model correctly learns that more coronary arteries with evidence of calcification indicate increased risk of disease. Additionally, Integrated Hessians reveals that our model learns a negative interaction between the number of coronary arteries with calcium accumulation and female gender. This supports the well-known phenomenon of under-recognition of heart disease in women -- at the same levels of cardiac risk factors, women are less likely to have clinically manifest coronary artery disease \citep{maas2010gender}.\\

\begin{figure}
    \centering

    \includegraphics[width=.95\columnwidth]{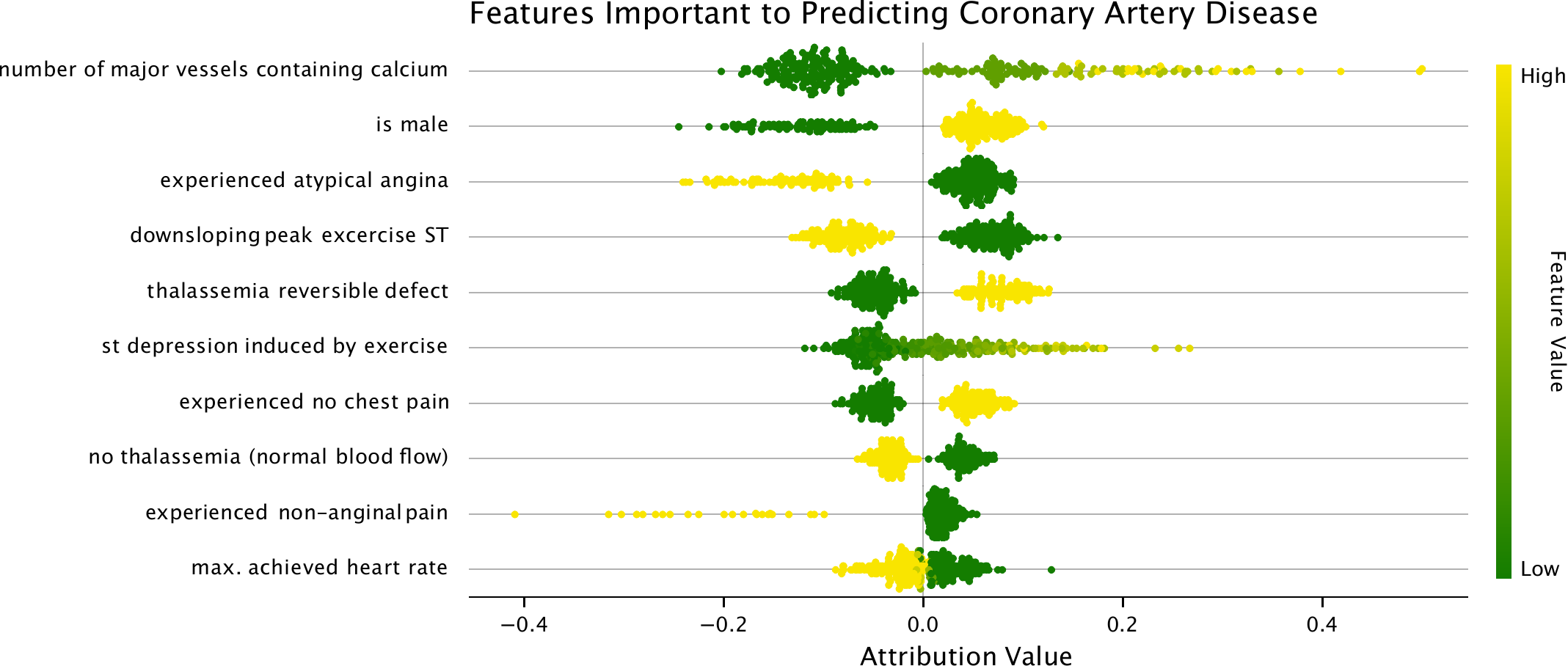}
    
    \caption{A summary of which features were most important towards predicting heart disease. A positive attribution value indicates increased risk of heart disease (negative value indicates decreased risk of heart disease). The features are ordered by largest mean absolute magnitude over the dataset. For binary features, high (yellow) indicates true while low (green) indicates false. For example, for the feature ``experienced atypical angina'', yellow means the patient did experience atypical angina and green means the patient did not. }
    
    \label{fig:heart_disease_summary}
\end{figure}

\begin{figure}
    \centering

    \includegraphics[width=.75\columnwidth]{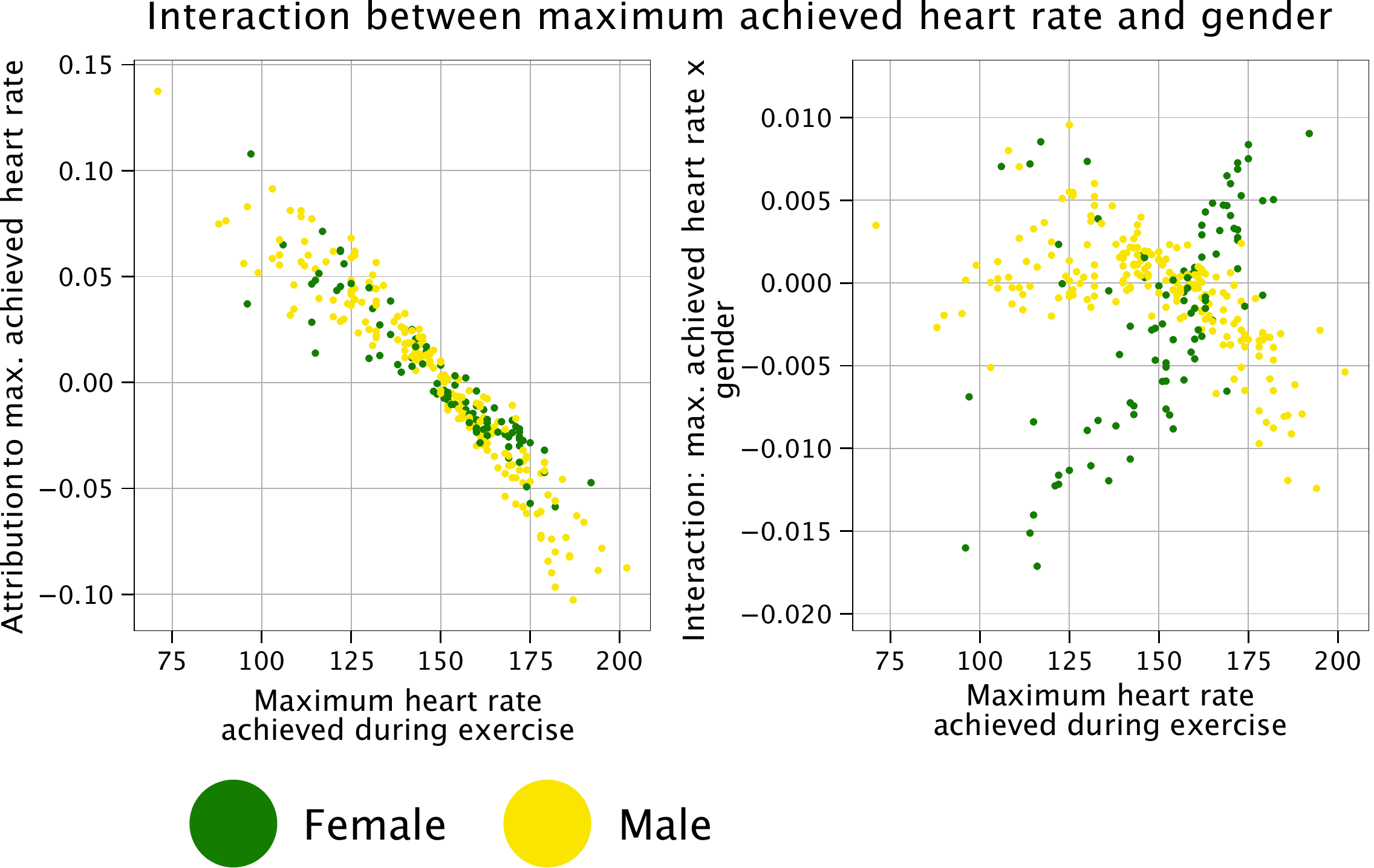}
    
    \caption{An interaction learned by the model between maximum achieved heart rate during exercise and the gender of the patient. In general, achieving a higher heart rate during exercise indicated lower risk of heart disease - but the model learns that this pattern stronger for men than it is for women.}
    
    \label{fig:max_heart_rate}
\end{figure}

\begin{figure}
    \centering

    \includegraphics[width=.75\columnwidth]{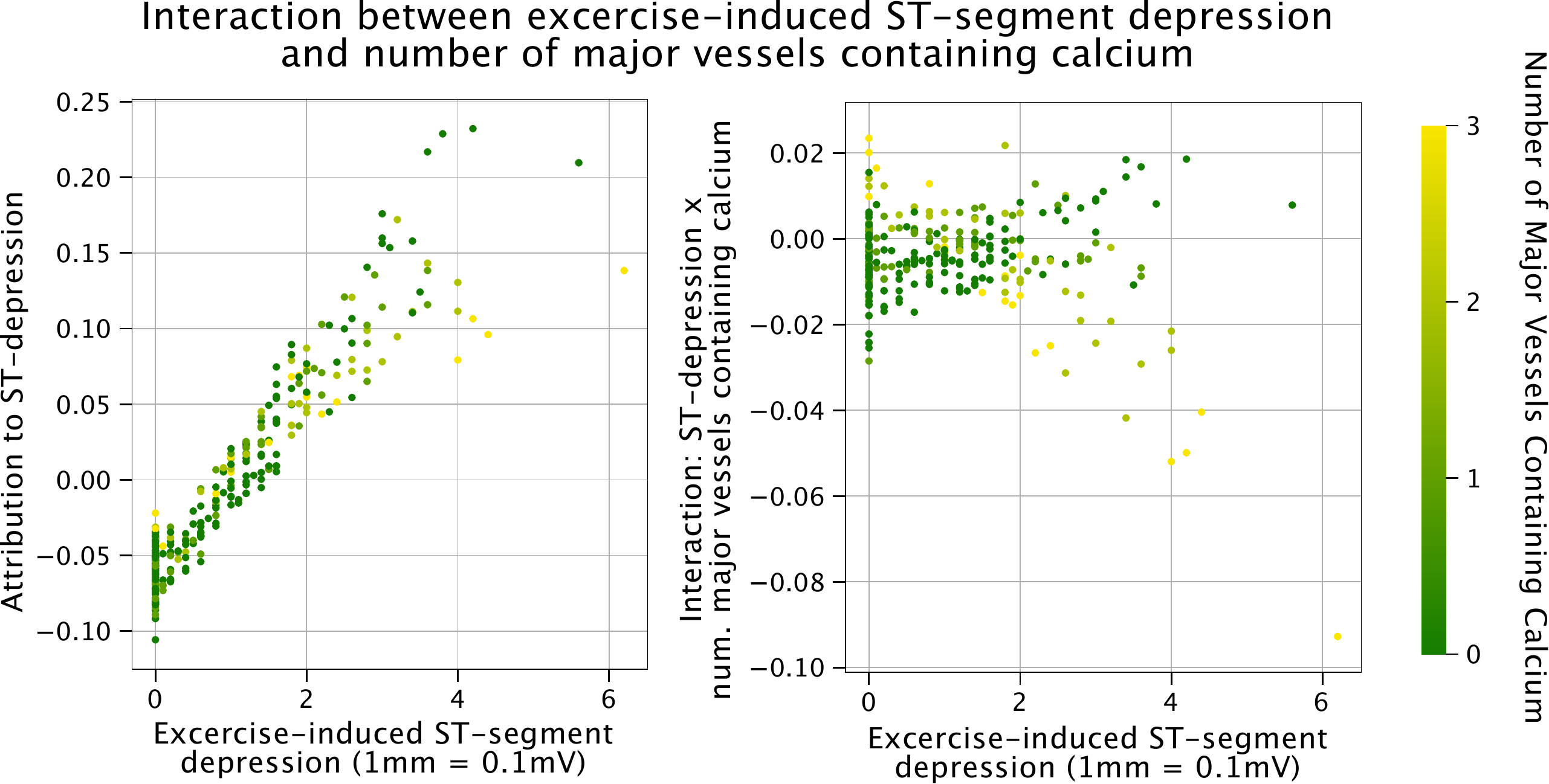}
    
    \caption{An interaction learned between ST-segment depression and number of major vessels appearing to contain calcium. The interaction seems to indicate that if a patient has many vessels appearing to contain calcium, then st-segment depression is less important toward driving risk, probably because the number of major vessels containing calcium becomes the main risk driver. }
    
    \label{fig:st_depression}
\end{figure}

\begin{figure}
    \centering
    \includegraphics[width=0.75\columnwidth]{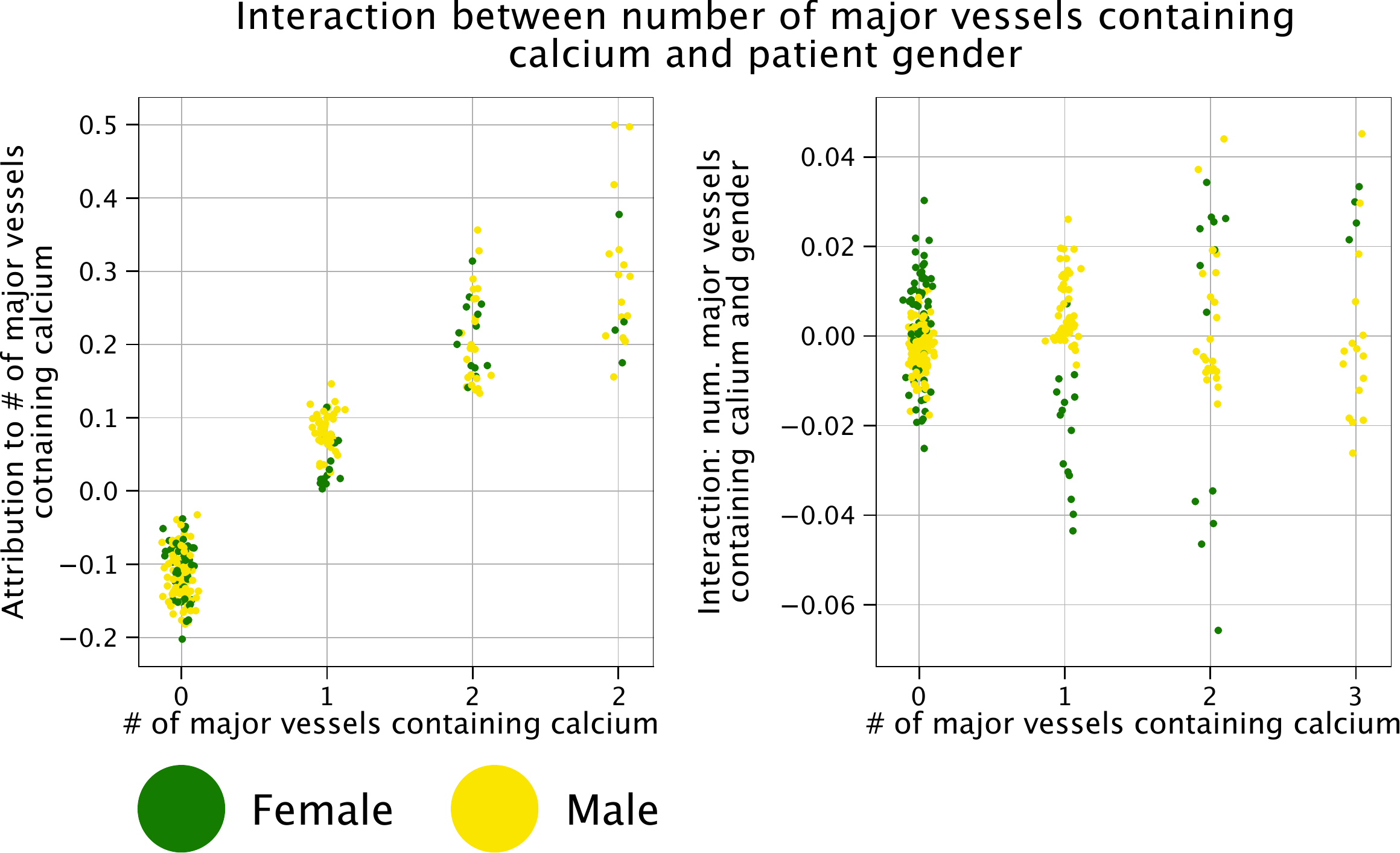}
    \caption{Left: Expected Gradients feature importance of the number of major vessels with accumulation of calcium as indicated by cardiac cinefluoroscopy. More vessels with calcium build-up indicated increased risk. Right: Expected Hessians feature interactions between patient gender and the number of major vessels containing calcium. When the Expected Hessians interactions are aggregated across the dataset, they reveal that our model has learned that women with calcium deposition in one coronary artery are less likely than men to be diagnosed with coronary artery disease.}
    \label{fig:major_vessels}
\end{figure}

\subsection{Pulsar star prediction}
\label{sec:pulsar}

In this section, we use a physics dataset to confirm that a model has learned global pattern that is visible in the training data. We utilize the HRTU2 dataset, curated by \citet{lyon2016fifty} and originally gathered by \citet{keith2010high}. The task is to predict whether or not a particular signal measured from a radio telescope is a pulsar star or generated from radio frequency interference (e.g. background noise). The features include statistical descriptors of measurements made from the radio telescope. The dataset contains 16,259 examples generated through radio frequency interference and 1,639 examples that are pulsars.  The dataset we use has 4 statistical descriptors - mean, standard deviation, skewness and kurtosis - of two measurements relating to pulsar stars: the integrated pulse profile (IP) and the dispersion-measure signal-to-noise ratio curve (DM-SNR), for a total of 8 features. The integrated pulse profile measures how much signal the supposed pulsar star gives off as a function of phase of the pulsar: as pulsars rotate, they emit radiation from their magnetic poles, which sweep over the earth periodically. We can measure the radiation over time using a radio telescope and aggregating measurements over phase to get the integrated pulse profile. Signals that are pulsars stars should in theory have stronger, more peaked integrated pulse profiles than those generated from radio frequency interference. The DM-SNR curve measures how much phase correction changes the signal-to-noise ratio in the measured signal. Since pulsars are far away, their radio emissions get dispersed as they travel from the star to earth: low frequencies get dispersed more than high frequencies (e.g. they arrive later). Phase correction attempts to re-sync the frequencies; however, no amount of phase correction should help peak a signal if the signal was generated from radio frequency interference rather than a legitimate pulsar.\\

On this task, we use a two-layer neural network with 32 hidden units in both layers and the softplus activation function after each layer. We optimize using stochastic gradient descent with a batch size of 256. We use an initial learning rate of 0.1 that decays with a rate of 0.96 every batch. We use nesterov momentum as well with $\beta = 0.9$ \cite{sutskever2013importance}. We train for 10 epochs and use a class-weight ratio of 1:3 negative to positive to combat the imbalance in the training dataset. Again, we note that these hyper-parameters are not necessarily optimal but were simply chosen because they produced reasonable convergence on the training set. We split the data into 14,318 training examples (1,365 are pulsars) and 3,580 testing examples (274 are pulsars), and achieve a held out test accuracy of 0.98 (0.86 TPR and 0.99 TNR).\\

To generate attributions and interactions for this dataset, we use Expected Gradients and Expected Hessians with the training set forming the background distribution. We use 200 samples to compute both attributions and interactions, although, as also noted on the previous section about the heart disease task, 200 samples was probably larger than necessary. In Figure \ref{fig:pulsar_star}, we examine the interaction between two key features in the dataset: kurtosis of the integrated profile, which we abbreviate as kurtosis (IP), and standard deviation of the dispersion-measure signal-to-noise ratio curve, which we abbreviate as standard deviation (DM-SNR). The bottom of Figure \ref{fig:pulsar_star} shows that kurtosis (IP) is a highly predictive feature, while standard deviation (DM-SNR) is less predictive. However, in the range where kurtosis (IP) is roughly between 0 and 2, standard deviation (DM-SNR) helps distinguish between a concentration of negative samples at standard deviation (DM-SNR) $<$ 40. We can verify that the model we've trained correctly learns this interaction. By plotting the interaction values learned by the model against the value of kurtosis (IP), we can see a peak positive interaction for points in the indicated range and with high standard deviation (DM-SNR). Interaction values show us that the model has successfully learned the expected pattern: that standard deviation (DM-SNR) has the most discriminative power when kurtosis (IP) is in the indicated range.

\begin{figure}
    \centering
    \includegraphics[width=0.75\columnwidth]{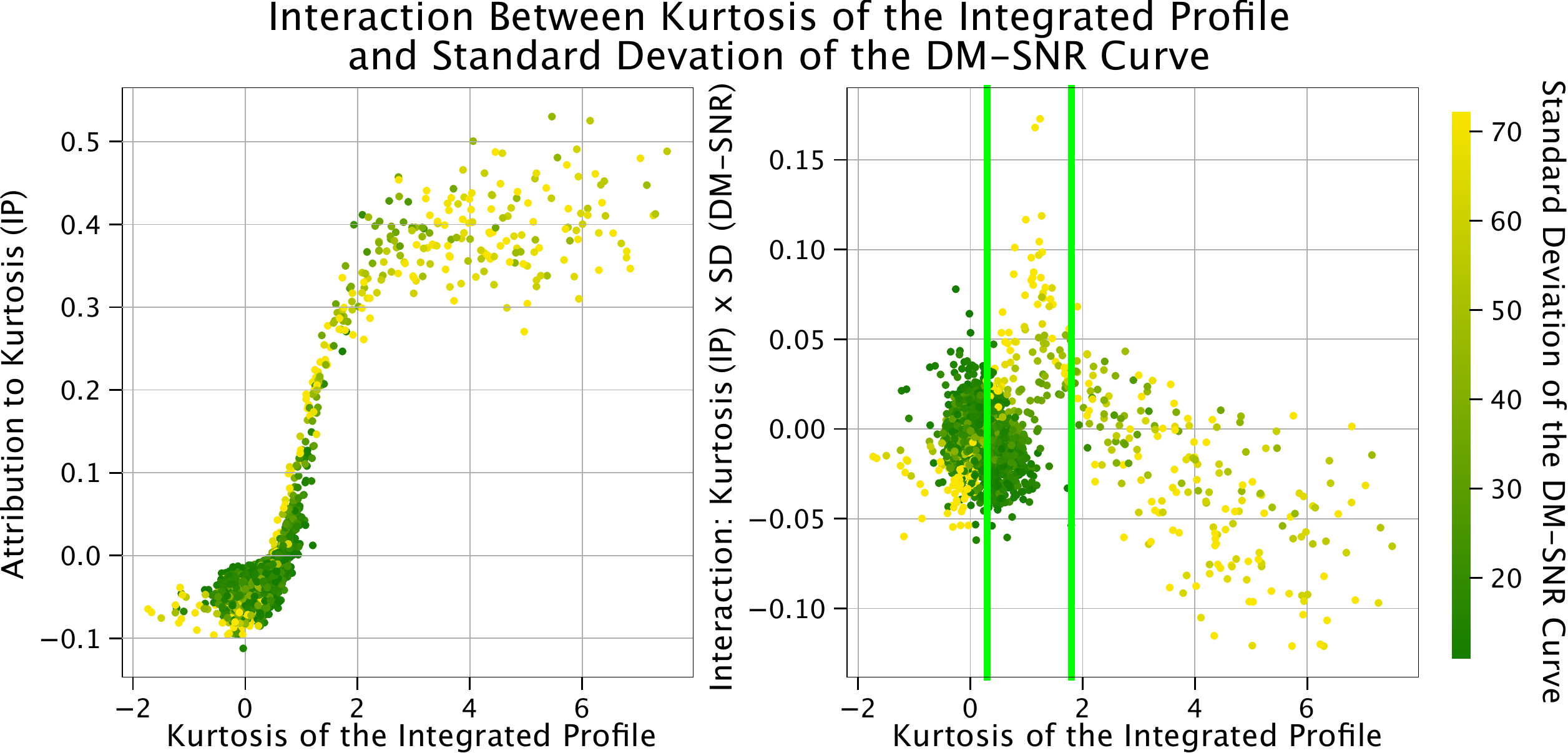}
    
    \vspace{0.25cm}
    
    \includegraphics[width=0.75\columnwidth]{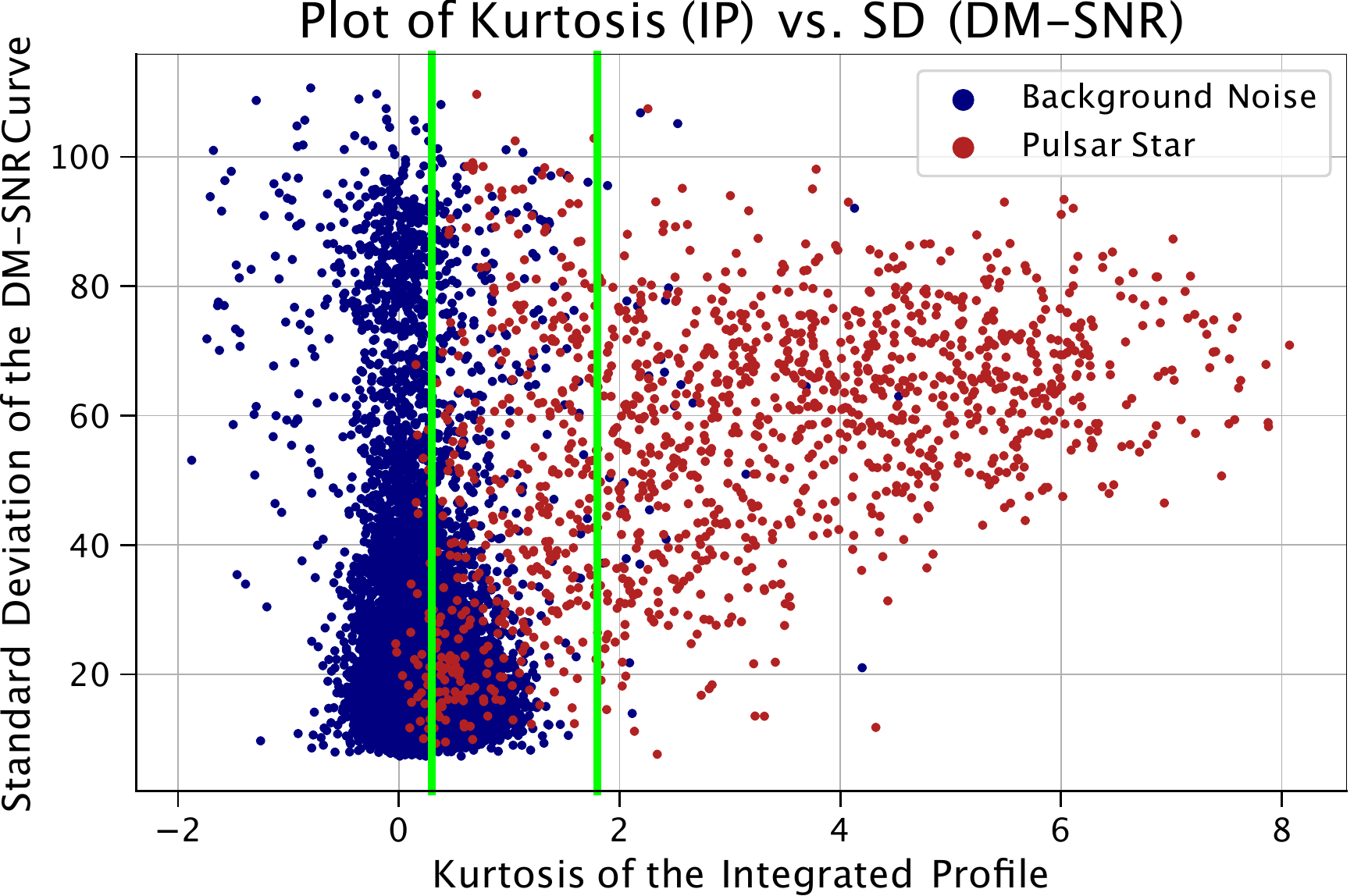}
    \caption{ Top Left: Attributions to kurtosis (IP) generated by expected gradients. Top Right: The model learns a peak positive interaction when kurtosis (IP) is in the range [0, 2]. Bottom: A plot of the training data along the axes of the two aforementioned features, colored by class label. Although kurtosis (IP) seems to be the more predictive feature, in the highlighted band the standard deviation (DM-SNR) provides useful additional information: larger standard deviation (DM-SNR) implies higher likelihood of being a pulsar star. 
    }
    \label{fig:pulsar_star}
\end{figure}

\section{Details for anti-cancer drug combination response prediction}

\subsection{Data description}

As mentioned in the main text, our dataset consists of 12,362 samples (available from http://www.vizome.org/). Each sample consists of the measured response of a 2-drug pair tested in the cancer cells of a patient \cite{tyner2018functional}. The 2-drug combination was described with both a drug identity indicator and a drug target indicator. For each sample, the drug identity indicator is a vector $x_{\textrm{id}} \in \mathbb{R}^{46}$ where each element represents one of the 46 anti-cancer drugs present in the data, and takes a value of $0$ if the corresponding drug is not present in the combination and a value of $1$ if the corresponding drug is present in the combination. Therefore, for each sample, $x_{\textrm{id}}$ will have 44 elements equal to 0 and 2 elements equal to 1. This is the most compact possible representation for the 2-drug combinations. The drug target indicator is a vector $x_{\textrm{target}} \in \mathbb{R}^{112}$ where each element represents one of the 112 unique molecular targets of the anti-cancer drugs in the dataset. Each entry in this vector is equal to $0$ if neither drug targets the given molecule, equal to $1$ if one of the drugs in the combination targets the given molecule, and equal to $2$ if both drugs target the molecule. The targets were compiled using the information available on DrugBank \citep{wishart2018drugbank}. The \textit{ex vivo} samples of each patient's cancer was described using gene expression levels for each gene in the transcriptome, as measured by RNA-seq, $x_\textrm{RNA} \in \mathbb{R}^{15377}$. Before training, the data was split into two parts -- 80\% of samples were used for model training, and an additional 20\% were used as a held-out validation set to determine when the model had been trained for a sufficient number of epochs.

\subsection{RNA-seq preprocessing}

The cancerous cells in each sample were described using RNA-seq data -- measurements of the expression level of each gene in the sample. We describe here the preprocessing steps used to remove batch effects while preserving biological signal. We first converted raw transcript counts to fragments per kilobase of exon model per million mapped reads (FPKM), a measure that is known to better reflect the molar amount of each transcript in the original sample than raw counts. FPKM accounts for this by normalizing the counts for different genes according to the length of transcript, as well as for the total number of reads included in the sample \citep{}. The equation for FPKM is given as:
\begin{equation}
        \textrm{FPKM} = \frac{X_i \times 10^9}{Nl_i},
\end{equation}
where $X_i$ is the vector containing the number of raw counts for a particular transcript $i$ across all samples, $l_i$ is the effective length of that transcript, and $N$ represents the total number of counts. After converting raw counts to FPKM, we opt to consider only the protein-coding part of the transcriptome by removing all non-protein-coding transcripts from the dataset. Protein-coding transcripts were determined according to the list provided by the HUGO Gene Nomenclature Committee (\href{https://www.genenames.org/download/statistics-and-files/}{https://www.genenames.org/download/statistics-and-files/}). In addition to non-protein-coding transcripts, we also removed any transcript that was not observed in $>70\%$ of samples. Transcripts are then $\log_2$ transformed and made 0-mean unit variance. Finally, the ComBat tool (a robust empirical Bayes regression implemented as part of the sva R package) was used to correct for batch effects \citep{10.1371/journal.pgen.0030161}.

\subsection{Model and training description}

\begin{figure}
    \centering

    \includegraphics[width=.7\columnwidth]{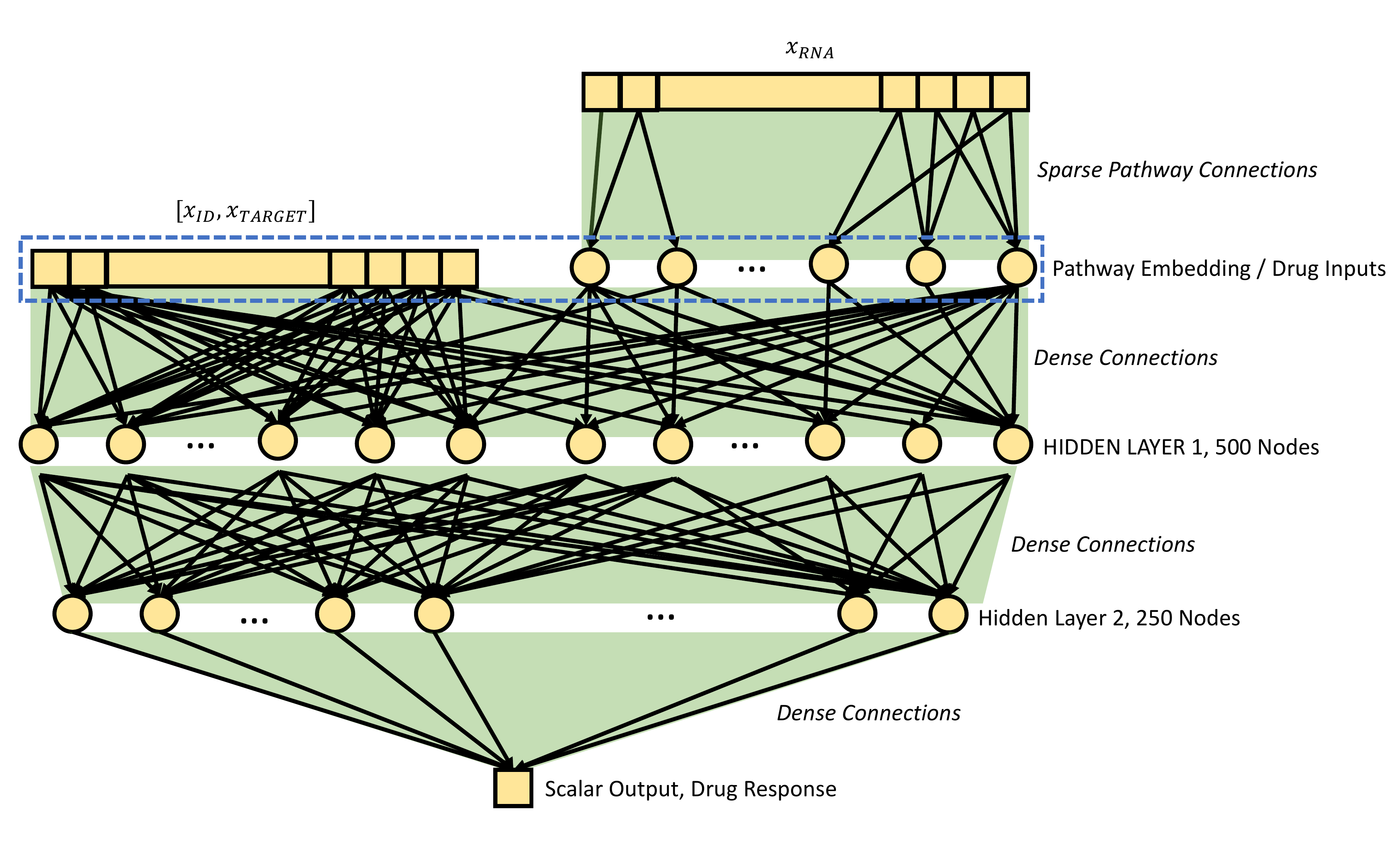}
    
    \caption{Neural network architecture for anti-cancer drug combination response prediction. We learn an embedding from all RNA-seq gene expression features ($x_{\textrm{RNA}}$) to KEGG pathways by sparsely connecting the inputs only to nodes corresponding to the pathways of which they are members. When we calculate feature attributions and interactions, we attribute to the layer that contains the raw drug inputs and the learned pathway embeddings (layer boxed with dashed blue line).}
    
    \label{fig:pathway_neural_net_architecture}
\end{figure}

To model the data, we combined the successful approaches of \citet{10.1093/bioinformatics/btx806} and \citet{hao2018pasnet}. Our network architecture is a simple feed-forward network (\autoref{fig:pathway_neural_net_architecture}), as in \citet{10.1093/bioinformatics/btx806}, where there were two hidden layers of 500 and 250 nodes respectively, both with Tanh activation. In order to improve performance and interpretability, we followed \citet{hao2018pasnet} in learning a \emph{pathway-level} embedding of the gene expression data. The RNA-seq data, $x_\textrm{RNA} \in \mathbb{R}^{15377}$, was sparsely connected to a layer of $1077$ nodes, where each node corresponded to a single pathway from KEGG, BioCarta, or Reactome \citep{kanehisa2002kegg,nishimura2001biocarta,croft2014reactome}. We made this embedding non-linear by following the sparse connections with a Tanh activation function. The non-linear pathway embeddings were then concatenated to the drug identity indicators and the drug target indicators, and these served as inputs to the densely connected layers.We trained the network to optimize a mean squared error loss function, and used the Adam optimizer in PyTorch with default hyperparameters and a learning rate equal to $10^{-5}$ \citep{kingma2014adam}. We stopped the training when mean squared error on the held-out validation set failed to improve over 10 epochs, and found that the network reached an optimum at $200$ epochs. For the sake of easier calculation and more human-intuitive attribution, we attribute the model's output to the layer with the pathway embedding and drug inputs, rather than to the raw RNA-seq features and drug inputs (see \autoref{fig:pathway_neural_net_architecture}).

For this experiment, we calculated all explanations and interactions using the Integrated Gradients and Integrated Hessians approach using the all zeros vector as reference and $k > 256$ interpolation points.

\subsection{Biological interaction calculation}

To evaluate how well the interactions detected by Integrated Hessians match with the ground truth for biological drug-drug interactions in this dataset, we can use additional single drug response data that our model was not given access to in order to calculate \emph{biological synergy}. Drug synergy is the degree of extra-additive or sub-additive response observed when two drugs are combined as compared to the additive response that would be expected if there were no interaction between the two compounds. The drug response for a single drug is measured as $\textrm{IC50}^{\textrm{single}}$, or the dose of that single drug necessary to kill half of the cells in an \textit{ex vivo} sample. The drug response for a drug combination is measured as $\textrm{IC50}^{\textrm{combination}}$, or the dose of an equimolar combination of two drugs necessary to kill half of the cells in an \textit{ex vivo} sample. The drug synergy between two drugs $a$ and $b$ can be calculated using the $CI$, or combination index:

\begin{equation}
    CI_{a,b} = \frac{\textrm{IC50}_a^{\textrm{combination}}}{\textrm{IC50}_a^{\textrm{single}}} + \frac{\textrm{IC50}_b^{\textrm{combination}}}{\textrm{IC50}_b^{\textrm{single}}}.
\end{equation}

For $CI$, a value greater than 1 indicates anti-synergy (negative interaction) while a value less than 1 indicates synergy (positive interaction). While our model was trained solely to predict $\textrm{IC50}^{\textrm{combination}}$, we can see how well the model learned true biological interactions by using the additional single drug response data to calculate synergy for particular samples. As described in the main text, when we calculated $CI$ values for all samples in which the combination of Venetoclax and Artemisinin was tested, then binarized the samples into synergistic and anti-synergistic, we can see that the Integrated Hessians values were higher in the truly synergistic group than in the truly anti-synergistic group ($p = 2.31 \times 10^{-4}$). We note that this is particularly remarkable, given that the model had no access to single drug data whatsoever.

\end{document}